\documentclass{article}

\IfFileExists{microtype.sty}{\usepackage{microtype}}{}
\usepackage{graphicx}
\usepackage{booktabs}
\usepackage{float}

\usepackage[dvipsnames]{xcolor}

\usepackage{hyperref}
\hypersetup{
    colorlinks=true,
    linkcolor=blue,
    filecolor=blue,
    urlcolor=blue,
    citecolor=purple,
}

\usepackage[accepted]{icml2026}

\usepackage{amsmath,amssymb,amsthm}
\usepackage{xspace}
\usepackage{enumitem}

\IfFileExists{pcrr7t.tfm}{}{}

\theoremstyle{plain}
\newtheorem{theorem}{Theorem}

\theoremstyle{definition}

\theoremstyle{remark}

\newcommand{\sgma}{\textsc{Sigma}\xspace}

\icmlsetsymbol{equal}{*}
\icmlsetsymbol{joint}{\textsuperscript{\textdagger}}

\icmltitlerunning{\sgma: Spectral Insights for LLM Model Collapse}

\counterwithin{equation}{section}

\begin{document}

\twocolumn[
\icmltitle{\sgma: Scalable Spectral Insights for LLM Model Collapse}

\begin{icmlauthorlist}
\icmlauthor{Yi Gu}{nu,equal}
\icmlauthor{Lingyou Pang}{ucd,equal}
\icmlauthor{Xiangkun Ye}{bu}
\icmlauthor{Tianyu Wang}{jhu}
\icmlauthor{Jianyu Lin}{jhu}
\icmlauthor{Carey E. Priebe}{jhu,joint}
\icmlauthor{Alexander Aue}{ucd,joint}
\end{icmlauthorlist}

\icmlaffiliation{ucd}{Department of Statistics, University of California, Davis}
\icmlaffiliation{nu}{Department of Mathematics, Northwestern University}
\icmlaffiliation{bu}{Boston University}
\icmlaffiliation{jhu}{Department of Applied Mathematics and Statistics, Johns Hopkins University}

\icmlcorrespondingauthor{Yi Gu}{Yi.Gu@u.northwestern.edu}
\icmlcorrespondingauthor{Lingyou Pang}{lyopang@ucdavis.edu}


\icmlkeywords{}

\vskip 0.3in
]

\printAffiliationsAndNotice{%
  \icmlEqualContribution \quad
  \texorpdfstring{\textsuperscript{\textdagger}}{}Alexander Aue and Carey E. Priebe jointly supervised this work and served as co-principal investigators.}

\begin{abstract}
\IfFileExists{sections/abstract.tex}{The rapid adoption of synthetic data for training Large Language Models (LLMs) has introduced the technical challenge of ``model collapse''--a degenerative process where recursive training on model-generated content leads to a contraction of distributional variance and representational quality. While the phenomenology of collapse is increasingly evident, rigorous methods to quantify and predict its onset in high-dimensional spaces remain elusive. In this paper, we introduce \sgma(\textbf{S}pectral \textbf{I}nequalities for \textbf{G}ram \textbf{M}atrix \textbf{A}nalysis), a unified framework that benchmarks model collapse through the spectral lens of the embedding Gram matrix. By deriving and utilizing deterministic and stochastic bounds on the matrix's spectrum, \sgma provides a mathematically grounded metric to track the contraction of the representation space. Crucially, our stochastic formulation enables scalable estimation of these bounds, making the framework applicable to large-scale foundation models where full eigendecomposition is intractable. We demonstrate that \sgma effectively captures the transition towards degenerate states, offering both theoretical insights into the mechanics of collapse and a practical, scalable tool for monitoring the health of recursive training pipelines.}{}
\end{abstract}

\IfFileExists{sections/background.tex}{\section{Introduction}
\label{sec:background}
\textbf{LLM Model Collapse}~\cite{shumailov2024,Shumailov2024-dj} refers to the phenomenon where Large Language Models (LLMs) progressively degrade towards a degenerate state. This state arises as a result of recursive training, where models are trained on iterations of data generated by other models, often previous versions of themselves. We term this model-generated output \textbf{Synthetic Data}, in contrast to organic data produced by human users.

In the current landscape of AI research, the demand for reliable and high-quality training data has never been higher. The scale of the latest generation of LLMs is increasing at such a staggering speed that the natural generation of internet content by humans cannot keep pace. Consequently, utilizing synthetic data has become a vital, if not inevitable, step in LLM training. Although this is a logical progression, it has a significant caveat. Without the infusion of new data produced by human users, model collapse occurs as the model iterates on its own probability distribution. Ultimately, the model may converge to a state characterized by low variance and degraded representational quality~\cite{shumailov2024, Shumailov2024-dj}.

Therefore, it is imperative to conduct a rigorous mathematical examination of LLM model collapse. What are the underlying mechanisms at the mathematical level? How can this degradation be tracked and quantified throughout the training process? Furthermore, how can we leverage these metrics to gain deeper insights into LLM dynamics and guide the development of more efficient training architectures? This paper aims to answer these questions.

\textbf{We introduce \sgma (\textbf{S}pectral \textbf{I}nequalities for \textbf{G}ram \textbf{M}atrix \textbf{A}nalysis), a novel framework that quantifies and measures model collapse by utilizing the spectral bounds of the Gram matrix of LLM embedding vectors.} The Gram matrix is a classical mathematical concept named after Jørgen Pedersen Gram. It has a wide range of applications in fields such as Riemannian geometry and control theory, and has found critical utility in modern applied mathematics and computer science domains, including the finite element method~\cite{10.5555/894170}. It is also widely used in classical NLP as an intuitive method to encode variable quantities of natural language embeddings into a fixed-dimensional matrix representation\cite{10.1162/153244302760200687,NIPS2014_b7866697}. As we enter the era of LLMs, the Gram matrix remains a vital tool, as it captures the intrinsic geometry and distributional properties of the data we seek to analyze. \sgma revisits this fundamental operator to quantify the modern phenomenon of representation collapse.

Inevitably, the spectrum of the Gram matrix becomes the key to understanding LLM model collapse. We develop our theory and framework around this central mathematical object, allowing us to monitor, estimate, and predict the behavior of LLMs regarding model collapse through training iterations. Crucially, our method utilizes both deterministic and stochastic bounds to function effectively even with partially observed data, ensuring our framework remains scalable for large-scale models. Moreover, this framework is black box friendly. While our experimental validation requires fine-tuning to simulate the recursion loop, the \sgma metrics themselves rely solely on output embeddings, allowing practitioners to track model collapse on deployed systems without accessing model weights or gradients.

Our contributions can be summarized as follows:
\begin{itemize}[leftmargin=*]
    \item \textbf{A Theoretical Framework for Collapse Detection:} We propose \sgma, a spectral analysis framework that rigorously links the decay of the Gram matrix eigenspectrum to the phenomenon of model collapse.
    \item \textbf{Rigorous Spectral Bounds:} We derive both deterministic and stochastic inequalities for the Gram matrix, allowing for precise estimation of representation collapse even in high-dimensional settings where exact computation is prohibitively expensive.
    \item \textbf{Scalable, Black Box-Friendly Metrics:} Leveraging these bounds, we introduce a set of scalable metrics that can track the ``health" of an LLM's distribution during recursive training, providing an early warning system for the onset of degenerate states, without the need of LLM internals.
\end{itemize}

\section{Background}

\subsection{Related Work}
The risk of training generative models on their own outputs has been formalized as \emph{model collapse}, where repeated self-training progressively erodes distributional support, particularly in the tails, and can culminate in highly generic or degenerate generations~\cite{shumailov2024,Shumailov2024-dj}.

Related phenomena have been studied under different labels and experimental protocols, including \emph{Model Autophagy Disorder} (MAD), which emphasizes precision--recall tradeoffs and the role of ``fresh'' real data injections across generations~\cite{alemohammad2023selfconsuminggenerativemodelsmad}.
Recent theory sharpens these concerns by distinguishing regimes where degradation is transient versus \emph{strong} and asymptotic, even when synthetic data is mixed with a constant fraction of real data~\cite{dohmatob2024strongmodelcollapse}.

Complementary analyses interpret recursion as an entropy- or variance-contracting process and propose tail-centric quantities such as time-to-forget and covariance shrinkage to explain why rare events vanish first~\cite{seddik2024badtrainingsyntheticdata}. Other lines formalize stability of self-consuming loops via recursion-stability notions, connecting collapse to unbounded generalization degradation across generations~\cite{yoon2025modelcollapseselfconsumingchain}.

Finally, practical collapse signals in text include over-concentration of surface features (e.g., $n$-grams), emphasizing the need for diagnostics that are sensitive to diversity loss beyond average quality metrics~\cite{zhu2025synthesizetextdatamodel}. Most prior work studies recursion within a single-model loop (or a fixed teacher--student chain), whereas real-world synthetic data arises from multi-model, retrieval-mediated ecosystems. Building on this black-box perspective, recent work develops statistically grounded uncertainty and calibration techniques tailored to LLM-as-a-Judge systems \cite{pang2025unsupervisedconformal,pang2025tamingvariability}, and studies related dynamics that can shift judge behavior over time \cite{wang2025llmwebdynamicstracing}, where
LLM Web Dynamics (LWD) explicitly targets this setting by modeling a network of LLMs coupled through a retrieval-augmented substrate, and analyzes convergence patterns under such web-like feedback.

In contrast to task- or token-level statistics alone, our \sgma framework focuses on \emph{representation geometry}: we track collapse via the eigenspectrum of embedding Gram matrices and develop deterministic and stochastic spectral inequalities that enable scalable, one-sided monitoring when only partial Gram information is available.

\subsection{Preliminaries}  \label{sec:prelim}

For initial information encoding in LLMs, the $i$-th sentence is mapped to a vector of a fixed dimension $m$. We denote this vector as $X_i \in \mathbb{R}^m$, which is referred to as the \textit{embedding} or the  \textit{embedding vector} of the sentence. Let the set of embedding vectors be $E_0 = \{X_i : i \in [n]\}$, where $[n]:= \{1,\dots,n\}$. $E_0$ represents our starting evaluation corpus. We operate under the regime where the sample size exceeds the embedding dimension, i.e., $n > m$.

Subsequently, we construct an $m \times n$ matrix $M$ by concatenating all embedding vectors column-wise. We define the Gram matrix as:
\begin{align}
    G = MM^{\!\top},
\end{align}
with respect to the set $E_0$. Note that $G$ is an $m \times m$ matrix whose dimensions are independent of the sample size $|E_0|$, though its entries depend on the samples. Since we assume $n > m$, it is possible for $G$ to be full rank. Given that LLMs are probabilistic models, the set $E$ contains random vectors; consequently, $G$ is treated as a random matrix. Our objective is to establish a framework using the spectrum of the Gram matrix to measure the magnitude of \textit{model collapse}.

It is well known that the Gram matrix $G$ captures the inter-dependency relations between the semantic features of the sentences. We posit that $G$ being full rank is a necessary condition for the LLM to function properly, indicating that the model utilizes the full capacity of its embedding space. Conversely, if $G$ is singular (or effectively rank-deficient), it implies information loss at the semantic level and marks the onset of model collapse. 

To provide an intuitive (non-rigorous) example: imagine a collapsed LLM that outputs identical or highly repetitive phrasing for distinct queries. These outputs would map to linearly dependent embedding vectors, failing to span the ambient space $\mathbb{R}^m$, thereby resulting in a singular Gram matrix.

Motivated by this, we introduce the log-determinant as our primary metric for model collapse: 
\begin{align} \label{eq:Gram_basic}
    \log |G| = \log \left(\prod_{j=1}^m \lambda_j\right) = \sum_{j=1}^{m} \log (\lambda_j),
\end{align}
where $\lambda_j$ are the eigenvalues of $G$. When collapse occurs, we expect a sharp decrease in the log-determinant (diverging towards $-\infty$ as $\lambda_{min} \to 0$). }{\section{Background}}
\IfFileExists{sections/method.tex}{
\section{The \sgma Framework: Deterministic \& Spectral Bounding of the Gram Matrix}
\label{sec:method4}
To rigorously quantify model collapse, we analyze the spectral properties of the embedding Gram matrix. However, computing the full spectrum for large-scale LLMs is often computationally intractable due to the sheer volume of tokens ($n_k$) relative to the embedding dimension ($m$). In this section, we present the \sgma framework, which leverages a sub-sampling strategy to derive both deterministic and stochastic bounds on the Gram determinant. This allows us to estimate the ``health'' of the representation space using only a fraction of the data.

\subsection{Setup and Notations}
\label{sec:method4:setup}
We track the model's state at a fixed training iteration $k$. Adding superscripts to the basic notation in \eqref{eq:Gram_basic}, let $M^{(k)} \in \mathbb{R}^{m \times n_k}$ be the embedding matrix, where each column $v^{(k)}_j \in \mathbb{R}^m, 1 \leq j \leq n_k$ is an embedding vector. We define the \emph{Gram matrix} as:
\begin{align}
    G^{(k)}:=M^{(k)} \left(M^{(k)}\right)^{\!\top} = \sum_{j=1}^{n_k} v^{(k)}_j \left(v^{(k)}_j\right)^{\!\top}\in \mathbb{S}^m_{+},
\end{align}
where $\mathbb{S}^m_{+}$ denotes the set of $m \times m$ positive semi-definite matrices. We denote the eigenvalues of $G^{(k)}$ by $\lambda_1(G^{(k)}) \ge \dots \ge \lambda_m(G^{(k)}) \ge 0$.

\textbf{The Sub-Sampling Strategy.}
Since $n_k$ is typically massive, we partition the embedding matrix $M^{(k)}$ into two blocks:
\begin{align}
    &A^{(k)} \in \mathbb{R}^{m \times n_A} \text{: The \textbf{observed} block;} \nonumber \\
    &B^{(k)} \in \mathbb{R}^{m \times n_B} \text{: The \textbf{unobserved} block.} \nonumber
\end{align}
Here, $n_k = n_A + n_B$, and we assume $n_A > m$ to ensure rank stability. $A^{(k)}$ consists of the columns we keep during the computation and $B^{(k)}$ contains the part we ignore. Write
\begin{align}
    G_A^{(k)} = A^{(k)}(A^{(k)})^\top, \quad G_B^{(k)} = B^{(k)}(B^{(k)})^\top
\end{align}
the \textbf{sub-Gram Matrices} formed by $A^{(k)}$ and $B^{(k)}$. A simple computation gives us
\begin{align}
    G^{(k)} \ = \ G_A^{(k)} + G_B^{(k)}.
\end{align}
Our goal is to bound the spectrum (and determinant) of the full $G^{(k)}$ using only the observable $G_A^{(k)}$.

\subsection{Main Results} \label{sec:theorems}

We first provide a rigorous deterministic bound that relies on no distributional assumptions, utilizing Weyl's inequality and Ky Fan dominance. The detailed deduction as well as the proofs of theorems can be found in appendix \ref{app:math}.

\begin{theorem}[Deterministic Spectral Bound]
\label{thm:deterministic}
Let $G^{(k)} = G_A^{(k)} + G_B^{(k)}$ be as defined above. Let $\beta_k := \lambda_{\max}(G_B^{(k)})$ be the spectral radius of the unobserved component. Then, the determinant of the full Gram matrix is bounded by:
\begin{align}
    \det(G_A^{(k)}) \ \leq \ \det(G^{(k)}) \ \leq \ \prod_{i=1}^m \left(\lambda_i(G_A^{(k)}) + \beta_k\right).
\end{align}
\end{theorem}
While Theorem~\ref{thm:deterministic} is rigorous, the dependence on $\beta_k$ (the max eigenvalue of the unobserved data) makes it loose in practice. To obtain a scalable estimator for real-world monitoring, we derive a stochastic bound assuming the data is drawn from a consistent underlying distribution. This assumption does not restrict the geometric complexity of the embeddings; it merely posits that the generative process is stationary, ensuring that the spectral properties of the sub-sample generalize to the full corpus.

\begin{theorem}[Stochastic Scaling Spectral Bound]
\label{thm:stochastic}
Assume the columns of $M^{(k)}$ are i.i.d. samples drawn from a distribution with positive-definite covariance matrix $C \in \mathbb{S}^m_{+}$. As $n_A, n_k \to \infty$ with fixed $m$, for any $\delta \in (0, 1)$, there exists a constant $K$ such that with probability at least $1-\delta$:
\begin{align}
    \det(G^{(k)}) \le K \cdot \left(\frac{n_k}{n_A}\right)^m \det(G_A).
\end{align}
\end{theorem}
Note that we assume $C$ is positive definite for convenience. In practice, if $C$ is singular, we instead consider the regularized matrix $C+\delta I_m$, which is exactly what we do in section \ref{subsec:method4:ubmetrics}.

Theorem \ref{thm:stochastic} implies a \textbf{simple operational scaling law}: for a healthy model, the log-determinant increases naturally with sample size as $m \log n_k$.
This allows us to define a \emph{size-invariant} baseline: the difference $\log\det G^{(k)} - m\log n_k$ should remain constant across checkpoints.
By monitoring this corrected value, we isolate genuine geometric collapse from trivial fluctuations in dataset size, ensuring that any observed negative drift reflects a true degradation in the model's representational capacity.

\subsection{Error Estimation}
Naturally, we are interested in how much error the theorems in the previous section introduce. In fact, theorem \ref{thm:stochastic} provides extra insights for the tightness of inequalities in \ref{thm:deterministic}. Let $C$ be the covariance matrix defined in theorem \ref{thm:stochastic}. We know that for large $n_A, n_k$, $\det(G^{(k)})\approx\left( n_k/n_A \right)^m \det(G_A^{(k)}) \approx n^m \det (C)$, we can define the asymptotic \textbf{Overestimation Ratio} $\mathcal{R}$ as the ratio of the Weyl bound to the stochastic expected value:
\begin{align}
    \mathcal{R} &= \frac{1}{\det G^{(k)}} \prod_{i=1}^m \left(\lambda_i(G_A^{(k)}) + \beta_k\right) 
    \nonumber \\
    &\approx \prod_{j=1}^m \frac{n_A \lambda_j(C) + (n_k-n_A) \lambda_{\max}(C)}{n \lambda_j(C)}.
\end{align}
Dividing through by $n$, and letting $\rho = n_A/n$, we have a more concise formula
\begin{align} \label{eq:overestimation_ratio}
    \mathcal{R} \approx \prod_{j=1}^m \left( \rho + (1-\rho) \frac{\lambda_{\max}(C)}{\lambda_j(C)} \right).
\end{align}

For each $j$, we consider the ratio $\lambda_{\max}(C) / \lambda_j(C) \ge 1$. If this ratio is close to $1$ for all $j$, it implies the covariance matrix $C$ is close to being \textit{isotropic}; consequently, the overestimation ratio will also be close to $1$, yielding an asymptotically tight deterministic bound.

On the other hand, if $C$ is more \textit{anisotropic} - i.e., the ratio $\lambda_{\max}(C) / \lambda_j(C) \gg 1$ for some smaller eigenvalues $\lambda_j(C)$ - then the term $(1-\rho) \lambda_{\max}(C)/\lambda_j(C)$ will dominate the product in \eqref{eq:overestimation_ratio} and cause the error to grow.

As for theorem \ref{thm:stochastic} itself, we can also give a more refined estimation of the error. Let $X$ be a random vector drawn from the i.i.d. distribution in theorem \ref{thm:stochastic}, we set
\begin{align} \label{eq:variance_defn}
    \sigma^2 := \text{Var} \left(X^TC^{-1}X\right).
\end{align}
The following theorem holds:
\begin{theorem}[Stochastic Error Term] \label{thm:err_stochastic}
Assume the settings of theorem \ref{thm:stochastic} and, in addition, the distribution of column vectors has a finite fourth moment. Then as $n_A, n_k \to \infty$ with $n_A<n_k$, we have
\begin{align}
    \frac{\log \det G^{(k)} - \log \det G_A^{(k)} - m \log (n_k/n_A)}{\sigma \ \sqrt{\frac{1}{n_A}-\frac{1}{n_k}}} \to \mathcal{N}(0,1)
\end{align}
 in distribution. Consequently, for any $\alpha \in (0,1)$, let $z_{\alpha/2} = \Phi^{-1}(1-\alpha/2)$. Then with probability approaching $1-\alpha$:
 \begin{align} \label{eq:confi_interval}
     \left| \log \det G_n - \log \left( \left(\frac{n_k}{n_A}\right)^m \det G_{A} \right) \right| \quad \quad \quad & \nonumber \\
      \le z_{\alpha/2} \cdot \sigma  \sqrt{\frac{1}{n_A} - \frac{1}{n_k}}.&
 \end{align}
\end{theorem}

\textbf{Stability in Large Corpora.}
While theorem \ref{thm:stochastic} establishes the asymptotic scaling law, theorem \ref{thm:err_stochastic} quantifies the precision of this law in finite samples. Specifically, it provides the exact confidence intervals required to bound the error of our scaling-law estimator in the finite-sample regime. The estimation error depends on the factor $\sqrt{1/n_A - 1/n_k}$. In the practical regime where the total volume of model outputs is massive compared to our observed sample ($n_k \gg n_A$), this factor simplifies effectively to
\begin{align}
\sqrt{\frac{1}{n_A} - \frac{1}{n_k}}
\;\approx\;
\frac{1}{\sqrt{n_A}}.
\end{align}
This leads to a crucial practical insight: the precision of our collapse monitoring depends primarily on the \emph{size of our observed block} $n_A$, not on the total size of the unobserved corpus. Even if the total output $n_k$ grows indefinitely, our estimate remains stable.

This motivates the definition of a simple plug-in estimator for the full log-determinant:
\begin{align}
\widehat{L}^{(k)}
\;:=\;
\log\det G_A^{(k)} \;+\; m\log\left(\frac{n_k}{n_A}\right).
\end{align}
Using the rigorous result from \eqref{eq:confi_interval}, we derive a practical rule-of-thumb for large datasets ($n_k \gg n_A$) with $(1-\alpha)$ confidence interval:
\begin{align}
\log\det G^{(k)}
\approx
\widehat{L}^{(k)} \;\pm\; \frac{z_{\alpha/2}\,\sigma}{\sqrt{n_A}}.
\end{align}
This formula allows practitioners to easily determine the necessary sample size: to achieve a target precision $\varepsilon$, one simply needs to choose an observed block size of $n_A \gtrsim (z_{\alpha/2}\,\sigma / \varepsilon)^2$.

\subsection{\sgma-UB Collapse Monitoring}
\label{subsec:method4:ubmetrics}

Section~\ref{sec:theorems} provides two complementary routes to control the full Gram determinant:
a distribution-free deterministic bound (Theorem~\ref{thm:deterministic}) and a stationary/i.i.d.\ scaling law
(Theorem~\ref{thm:stochastic}). However, the raw statements are not directly usable and practical because (i) the
deterministic bound depends on the unobserved quantity $\beta_k=\lambda_{\max}(G_B^{(k)})$, and
(ii) the stochastic scaling contains unknown constants and is meaningful primarily up to trends.

Our goal in \sgma-UB is therefore \emph{operational}: to design online diagnostics that are
\textbf{fully computable} from the observed summary statistics $\big(G_A^{(k)},\, n_A,\, n_k\big)$,
while remaining \textbf{analytically trackable} across checkpoints. This bridge is a central strength
of the framework: even with \emph{partial Gram access} (only the sub-Gram $G_A^{(k)}$ plus sample counts),
we can still make \emph{principled inferences} about the full Gram geometry, with uncertainty
\emph{explicitly controlled} via our deterministic envelope and stochastic scaling with error guarantees.

\subsubsection{Metrics: two theorem-to-metric bridges}
\label{subsec:method4:metric_design}

The geometric quantity we ultimately care about is the log determinant induced by the full Gram matrix,
\[
L^{(k)}(\delta)\;:=\;\log\det\!\big(G^{(k)}+\delta I_m\big),
\]
where we fix a small $\delta>0$ and reuse the same $\delta$ across generations.
The $\delta I_m$ shift makes the log-determinant well-defined under rank deficiency and stabilizes
numerics near collapse (Appendix~\ref{app:sigma_ub}). The difficulty is that we never explicitly construct \(G^{(k)}\); instead, we only compute the observed sub-Gram matrix \(G_A^{(k)} = A^{(k)}(A^{(k)})^\top\). Our framework, however, provides a principled way to interpret and infer properties of the full Gram matrix.

\paragraph{Bridge I (Theorem~\ref{thm:deterministic}): a computable deterministic envelope.}
Theorem~\ref{thm:deterministic} implies that the unobserved component
can only \emph{inflate} the spectrum, yielding a one-sided control of the form
\begin{align}
G^{(k)} \;\preceq\; G_A^{(k)}+\beta_k I_m,
\end{align}
which implies
\begin{align}
L^{(k)}(\delta)\;\le\;\log\det\!\big(G_A^{(k)}+(\beta_k+\delta)I_m\big).
\end{align}
The obstacle is that $\beta_k=\lambda_{\max}(G_B^{(k)})$ is unobservable.

To remove $\beta_k$, we enforce a standard preprocessing/normalization so that every embedding column
satisfies $\|v_j^{(k)}\|_2^2\le \rho$. With $n_2:=n_k-n_A$ unseen columns, a worst-case tail-energy budget gives
the computable bound (Appendix~\ref{app:sigma_ub})
\begin{align}
\widehat{\beta}_k \;:=\; (n_k-n_A)\,\rho .
\label{eq:betahat_rho}
\end{align}
Substituting $\widehat{\beta}_k$ yields a valid one-sided envelope for $L^{(k)}(\delta)$. Finally, because the envelope level is dominated by the isotropic shift $(\widehat{\beta}_k+\delta)I_m$,
we normalize it to isolate the \emph{gain} contributed by the observed spectrum:
\begin{align}
\mathcal{G}_{\mathrm{KF}}^{(k)}(\delta)
&:= \log\det\!\left(G_A^{(k)}+(\widehat{\beta}_k+\delta)I_m\right)
\nonumber\\[-1mm]
&\qquad \quad \qquad- m\log(\widehat{\beta}_k+\delta).
\label{eq:metric_kf_gain}
\end{align}
Equivalently,
\begin{align}
\mathcal{G}_{\mathrm{KF}}^{(k)}(\delta)
=&\log\det\!\left(I_m+\frac{1}{\widehat{\beta}_k+\delta}G_A^{(k)}\right) \nonumber \\
=&\sum_{i=1}^m \log\!\left(1+\frac{\lambda_i(G_A^{(k)})}{\widehat{\beta}_k+\delta}\right),
\end{align}
so $\mathcal{G}_{\mathrm{KF}}^{(k)}(\delta)$ decreases when the observed eigenvalues contract.
This metric is deterministic and conservative: it requires no distributional assumptions, and its conservativeness is governed solely by an explicit upper bound on the unobserved tail spectral mass.

\paragraph{Bridge II (Theorem~\ref{thm:stochastic}): size correction and covariance normalization.}
Under the stationary/i.i.d.\ assumption of Theorem~\ref{thm:stochastic}, the full Gram determinant scales
approximately like
\[
\det(G^{(k)})\approx\left(\frac{n_k}{n_A}\right)^m \det(G_A^{(k)}),
\]
i.e., the log-determinant shifts by the purely geometric term $m\log(n_k/n_A)$.
Moreover, when $G_A^{(k)}\approx n_A C$ for a stable covariance $C$, we have
$\log\det(G_A^{(k)})\approx m\log n_A + \log\det(C)$.
This motivates a size-invariant proxy for the underlying covariance geometry:
\begin{align}
\mathcal{U}_{\mathrm{LLN,cov}}^{(k)}(\delta)
&:= \log\det\!\left(G_A^{(k)}+\delta I_m\right) - m\log n_A .
\label{eq:metric_lln_cov}
\end{align}
Intuitively, $\mathcal{U}_{\mathrm{LLN,cov}}^{(k)}(\delta)$ tracks $\log\det(C)$:
it removes the trivial $m\log n_A$ growth, so sustained negative drift reflects a genuine contraction of the
representation geometry rather than a change in sample size. While the absolute scaling in
Theorem~\ref{thm:stochastic} involves an unknown constant, this constant cancels when we monitor trends.

To remove checkpoint-independent constants and emphasize temporal contraction, we report drift relative to
a baseline checkpoint $k=0$:
\begin{align}
\Delta\mathcal{G}_{\mathrm{KF}}^{(k)}(\delta)
&:= \mathcal{G}_{\mathrm{KF}}^{(k)}(\delta)-\mathcal{G}_{\mathrm{KF}}^{(0)}(\delta),\nonumber\\
\Delta\mathcal{U}_{\mathrm{LLN,cov}}^{(k)}(\delta)
&:= \mathcal{U}_{\mathrm{LLN,cov}}^{(k)}(\delta)-\mathcal{U}_{\mathrm{LLN,cov}}^{(0)}(\delta).
\label{eq:diag_deltas}
\end{align}

\subsubsection{Online monitoring procedure}
\label{subsec:method4:monitoring}

To monitor the health of the representation space continuously during training, we utilize a partial observation strategy. At each checkpoint $k$, we compute the sub-Gram matrix $G_A^{(k)}$ on a fixed observed block size $n_A>m$. This allows us to evaluate both the Track-I (deterministic) and Track-II (stochastic) diagnostics using only the summary statistics $\big(G_A^{(k)},n_A,n_k\big)$, avoiding the prohibitive cost of full eigendecomposition.

Algorithm~\ref{alg:sigma_ub} details the computational steps. In practice, the $\log\det(\cdot)$ terms are computed efficiently via Cholesky decomposition on the shifted positive semi-definite (PSD) matrices (see Appendix~\ref{app:sigma_ub}).

\begin{algorithm}[t]
\caption{\sgma-UB Collapse Monitoring (upper-bound diagnostics)}
\begin{algorithmic}[1] \label{alg:sigma_ub}
\REQUIRE $\{G_A^{(k)}\}_{k=0}^K$, counts $\{n_k\}_{k=0}^K$, fixed $n_A$, and hyperparameters $\delta>0$, $\rho>0$.
\STATE Define $\textsc{LogDet}_\delta(X):=\log\det(X+\delta I_m)$.
\FOR{$k=0,\dots,K$}
  \STATE $\widehat{\beta}_k \gets (n_k-n_A)\rho$
  \STATE $S_k \gets \textsc{LogDet}_\delta(G_A^{(k)})$
  \STATE $\mathcal{G}_{\mathrm{KF}}^{(k)} \gets \textsc{LogDet}_\delta(G_A^{(k)}+\widehat{\beta}_k I_m) - m\log(\widehat{\beta}_k+\delta)$
  \STATE $\mathcal{U}_{\mathrm{LLN,cov}}^{(k)} \gets S_k - m\log n_A$
  \IF {$k=0$}
  \STATE store $(\mathcal{G}_{\mathrm{KF}}^{(0)},\mathcal{U}_{\mathrm{LLN,cov}}^{(0)})$
  \ELSE
  \STATE compute drifts $\Delta\mathcal{G}_{\mathrm{KF}}^{(k)}$, $\Delta\mathcal{U}_{\mathrm{LLN,cov}}^{(k)}$ via \eqref{eq:diag_deltas}
  \ENDIF
\ENDFOR
\end{algorithmic}
\textbf{Return:} $\Delta\mathcal{G}_{\mathrm{KF}}^{(K)}$, $\Delta\mathcal{U}_{\mathrm{LLN,cov}}^{(K)}$
\end{algorithm}

Both metrics are designed to decrease as the eigenvalues of $G_A^{(k)}$ contract, becoming increasingly sensitive as the spectrum compresses toward the regularization floor $\delta$. A sustained negative drift in either $\Delta\mathcal{G}_{\mathrm{KF}}^{(k)}$ or $\Delta\mathcal{U}_{\mathrm{LLN,cov}}^{(k)}$ signals the onset of collapse, while saturation at the floor suggests the model has reached a late-stage, rank-deficient state.

The dual-track design provides complementary theoretical insights. $\Delta\mathcal{G}_{\mathrm{KF}}$ acts as a \textbf{mathematically conservative bound}. Being deterministic, it guarantees a valid upper bound on collapse given a fixed tail-energy budget. However, its theoretical sharpness degrades when the embedding space is highly \textit{anisotropic}—i.e., dominated by a few prominent eigen-directions. In such cases, the scalar energy cap $\widehat{\beta}_k$ significantly overestimates the unobserved tail contribution, potentially masking early signs of contraction.

Conversely, $\Delta\mathcal{U}_{\mathrm{LLN,cov}}$ serves as a \textbf{statistical estimator}. By leveraging the Law of Large Numbers (LLN), it corrects for dataset size to estimate the intrinsic geometric density. Its limitation is statistical: if the underlying data distribution is highly complex or heavy-tailed relative to the block size $n_A$, the LLN convergence may be slow, introducing estimation variance. Consequently, a divergence where Track-II plummets while Track-I remains stable is often an informative signature of \textit{early collapse}—indicating that the geometry is contracting (visible to the sensitive probe) but has not yet violated the worst-case energy budget (Track-I).

}{\section{Methodology}}
\IfFileExists{sections/experiment.tex}{\section{Experiments}
\label{sec:experiments}

We empirically stress-test \emph{recursive synthetic self-training} under stable evaluation conditions to pinpoint \emph{where} collapse may arise in the data cycling loop. The central question is not whether recursion can degrade outputs, but whether the observed geometric contraction sufficiently represents and tracks the collapse mechanism across protocols. Throughout, we track collapse using the two \sgma-UB diagnostics introduced in Section~3.4, computed on a frozen evaluation protocol so that measured drift is attributable to model evolution rather than evaluation drift.

\smallskip

All experiments in the main text are run and trained on a single controlled benchmark dataset (\textsc{TECH})  from DecayBench. We focus on TECH (software \& technical domain contextual data) for two practical reasons: (i) it provides a large, heterogeneous set of prose spanning conceptual guides, tutorials, and operational documentation; and (ii)it provides a fixed evaluation prompt bank that can be reused unchanged across checkpoints, ensuring the distribution does not drift due to uncontrolled prompt-relevent variables. We construct a real TECH corpus $R$ as a collection of $N=1000$ fixed-length \emph{64-token contexts}. 
 Evaluation uses a \emph{frozen TECH prompt bank} $P$ with stable prompt IDs and a fixed bucket taxonomy (Creative, Divergent, Analogy/ELI5, What-if, Neutral). Every checkpoint is evaluated on the \emph{same} prompt IDs under \emph{fixed decoding} and embedded with a \emph{frozen} sentence encoder. This prevents “evaluation drift” (where changing prompts/decoding would create apparent metric changes unrelated to the model). To ensure we can form stable Gram matrices with embedding dimension $m=384$, we generate multiple stochastic responses per prompt at each checkpoint; the full evaluation protocol and seed schedule are specified in Appendix~\ref{app:exp_card_tech}.

\smallskip

At generation $g\ge 1$, we build a fresh synthetic pool $S^{(g)}$ by prompting the previous-generation model with each real context $x\in R$.  This produces one synthetic continuation per training context, so $|S^{(g)}|=|R|$ and token budgets remain comparable across generations. The training data at generation $g$ is a constant-budget mixture
of both \emph{synthetic} and \emph{real} data: a fraction $(1-\alpha)$ is drawn from $S^{(g)}$ and a fraction $\alpha$ from $R$, with the total number of 64-token training blocks held fixed to $|R|$.

\smallskip
\noindent\textbf{Two recursion settings.}
We study two matched settings that share similar regeneration protocol and mixture rule and differ only in whether model weights are reset between generations.

\emph{S1: restart-from-base (data recursion only).}
Each generation regenerates a fresh synthetic data pool, but training always restarts from the same base model at initial state:
\begin{align}
    M^{(g)} \leftarrow \mathrm{Finetune} \big(M^{(0)},\, D^{(g)}(\alpha)\big).
\end{align}
S1 isolates \emph{the impact of increasing synthetic data share in LLM training}: while the synthetic data distribution shifts across generations, there is no cumulative weight drift.

\emph{S2: true recursion (data + weight recursion).}
In this setting, we \emph{do not restart} from a fixed base checkpoint. Instead, we carry the model
weights forward across generations and finetune iteratively under the same regeneration and mixture rule:
\begin{align}
M^{(g)} \leftarrow \mathrm{Finetune}\big(M^{(g-1)},\, D^{(g)}(\alpha)\big).
\end{align}
S2 is included because it matches the \emph{real-world self-consuming pipeline}: each generation both
\emph{produces} the next synthetic data and \emph{inherits} the parameters shaped by earlier generations.
This induces a feedback loop that can magnify small changes over time. In addition, S2 and S1 comparision isolates the incremental impact of \emph{weight recursion}: it tests whether repeatedly updating parameters on recursively generated data produces extra (and possibly accelerating) contraction, beyond the drift caused by \emph{data feedback alone} under restart-from-base training.

\begin{figure*}[t]
  \centering
  \includegraphics[width=0.90\textwidth]{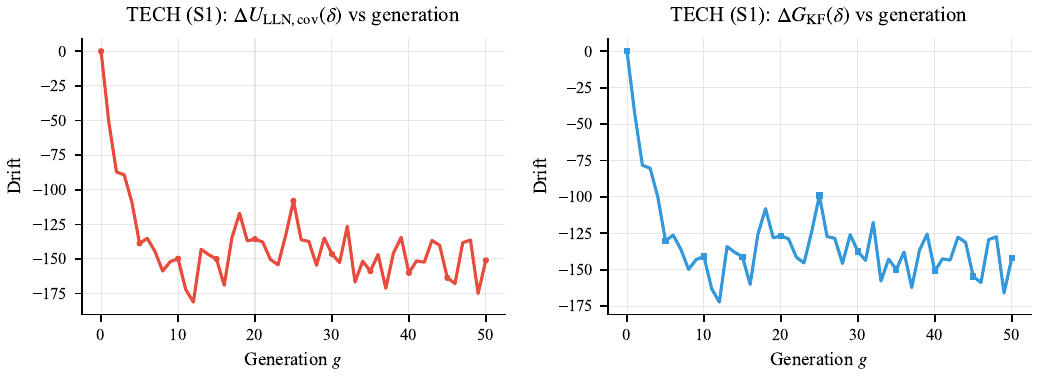}
  \caption{
  \textbf{TECH, Setting S1 (restart-from-base): drift vs.\ generation under both SIGMA-UB tracks.}
  \emph{Top:} Track~II $\Delta U_{\mathrm{LLN,cov}}(\delta)$.
  \emph{Bottom:} Track~I $\Delta G_{\mathrm{KF}}(\delta)$.
  Values are baseline drifts relative to the base checkpoint ($g=0$).
  S1 isolates \emph{data recursion only}: synthetic text is regenerated each generation, but training restarts from the same base checkpoint.
  Appendix~\ref{app:plot_tables_tech} reports the exact per-generation table used to render this figure.
  }
  \label{fig:tech_s1_tracks}
\end{figure*}

\begin{figure*}[t]
  \centering
  \includegraphics[width=0.90\textwidth]{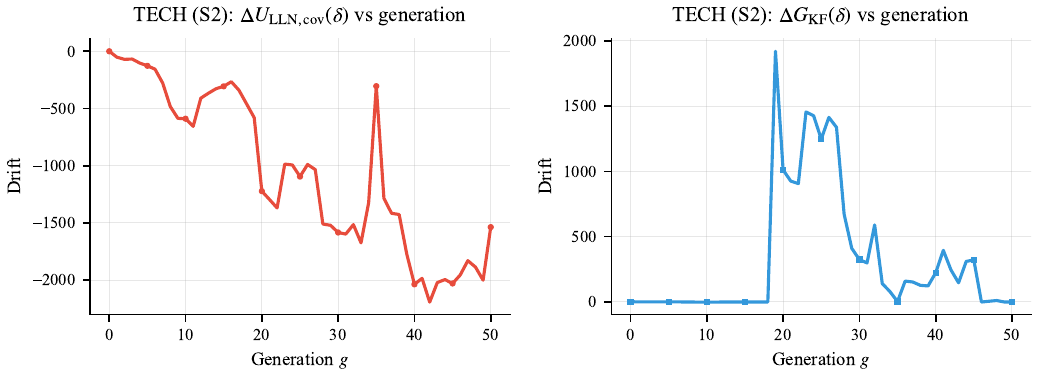}
  \caption{
  \textbf{TECH, Setting S2 (true recursion): drift vs.\ generation under both \sgma-UB tracks.}
  \emph{Top:} Track~II $\Delta U_{\mathrm{LLN,cov}}(\delta)$.
  \emph{Bottom:} Track~I $\Delta G_{\mathrm{KF}}(\delta)$.
  Values are baseline drifts relative to the base checkpoint ($g=0$).
  S2 compounds \emph{data recursion + weight recursion}: each generation regenerates synthetic text and continues training from the previous generation weights.
  Appendix~\ref{app:plot_tables_tech} reports the exact per-generation table used to render this figure.
  }
  \label{fig:tech_s2_tracks}
\end{figure*}

\begin{table*}[ht!]
\centering
\caption{TECH summary statistics (S1 vs S2). Final drifts and OLS slopes are computed from the plotted drift curves.}
\label{tab:tech_s1s2_summary}
\small
\begin{tabular}{lrrrrrr}
\toprule
Setting & g\_max & DeltaU\_final & slope\_LLN,cov & DeltaG\_final & slope\_KF & Track agreement? \\
\midrule
S1 & 50 & -150.986 & -0.941 & -142.051 & -0.918 & yes (both negative) \\
S2 & 50 & -1536.562 & -42.621 & -2.426 & 2.350 & no (Track\textasciitilde{}II collapses; Track\textasciitilde{}I does not) \\
\bottomrule
\end{tabular}
\end{table*}

\noindent\textbf{Result (TECH; S1 vs.\ S2).}
Restart-from-base recursion (S1) exhibits a mild but consistently negative drift under both tracks through $g_{\max}=50$ (final drifts $\approx-151$ and $\approx-142$), consistent with gradual contraction induced by the evolving synthetic pool even when weights are reset each generation. Under true recursion (S2), Track~II contracts by more than an order of magnitude (final $\approx-1537$; slope $\approx-42.6$), indicating substantially accelerated observed-spectrum contraction when weights are carried across generations. Track~I, in contrast, remains near zero and can even trend slightly positive over the same window; this reflects that the KF envelope is governed by a worst-case tail-energy budget rather than the observed sub-spectrum itself. Accordingly, Track~II/Track~I separation is an intended diagnostic signature of \sgma-UB and is operationally useful for flagging “early collapse” regimes where the observed spectrum has contracted sharply while the conservative envelope remains permissive.

\noindent For a more in-depth discussion, in the S2 loop (true recursion), early generations behave like self-distillation: the model is repeatedly trained to imitate its own sampled continuations, which quietly removes low-probability semantic/style alternatives even before outputs look obviously repetitive, so $\Delta\mathcal{U}_{\mathrm{LLN,cov}}(\delta)$ drifts negative from the start. As weight carryover accumulates, the model’s conditional distributions sharpen and fixed top-$p$/temperature sampling becomes \emph{effectively} more deterministic, so around $g\!\approx\!20$ the synthetic pool begins losing whole modes and $\Delta\mathcal{U}_{\mathrm{LLN,cov}}(\delta)$ drops much faster (Fig.~2). 

\noindent In contrast, $\Delta\mathcal{G}_{\mathrm{KF}}(\delta)$ reports the observed geometry only relative to a worst-case ``unseen energy'' budget, so modest early diversity loss can be masked and the curve stays near zero. The large positive excursion near $g\!\approx\!20$ is consistent with collapse concentrating probability mass into a few high-confidence templates. This can transiently \emph{inflate} a small number of dominant embedding directions (higher-energy but lower-diversity responses) and thus increase $\mathcal{G}_{\mathrm{KF}}$ even while global diversity is deteriorating. The resulting separation between $\Delta\mathcal{U}_{\mathrm{LLN,cov}}(\delta)$ and $\Delta\mathcal{G}_{\mathrm{KF}}(\delta)$ is therefore an interpretable early-collapse signature: breadth is shrinking first, and the conservative envelope only tightens once recursion crosses a tipping point.}{\section{Experiment}}
\IfFileExists{sections/conclusion.tex}{\section{Conclusion} \label{sec:conclusion}

In this work, we established \sgma as a scalable framework for monitoring the onset of model collapse without the prohibitive cost of full eigen-decomposition. By deriving rigorous spectral inequalities, we demonstrated that the health of an LLM’s representation space can be accurately estimated using only a sub-sample of the embedding Gram matrix, effectively decoupling monitoring costs from the total corpus size.

Our empirical analysis of recursive loops reveals a critical distinction between data-only recursion (S1) and true recursion with weight carryover (S2). We showed that while data recursion induces gradual drift, the compounding effect of weight updates accelerates geometric contraction by orders of magnitude. Crucially, the divergence between our conservative deterministic envelope (Track I) and the sensitive stochastic scaling metric (Track II) serves as a diagnostic signature for early-stage collapse, often preceding surface-level repetition artifacts.

\paragraph{Future Work} 
This paper marks a rigorous first step in quantifying the static geometry of model collapse. However, the recursive training loop is inherently a dynamical process--a trajectory of distributions evolving over time. Our future work aims to transition from \textit{monitoring} states to \textit{predicting} dynamics by viewing the LLM training loop through the lens of \textbf{Koopman Operator Theory}~\cite{brunton2021modernkoopmantheorydynamical}. 

While the state-space evolution of model parameters is highly non-linear and opaque, the Koopman operator lifts these dynamics into an infinite-dimensional function space where the evolution becomes linear. This powerful shift allows us to deploy the robust tools of functional analysis and ergodic theory to analyze the system's spectral properties. By characterizing the eigenvalues of the Koopman operator associated with the recursive training map, we hope to rigorously predict asymptotic behaviors--such as the rate of convergence to degenerate fixed points--and identify stable ``eigen-distributions'' that resist collapse. This approach promises not only to demystify the black-box nature of iterative learning but also to bridge the gap between modern dynamical systems theory and the empirical realities of foundation model maintenance.
}{\section{Conclusion}}


\section*{Impact Statement}

This paper advances the understanding of Large Language Models and the field of Machine Learning. There are many potential societal consequences
of our work, none of which we feel must be specifically
highlighted here.

\bibliographystyle{icml2026}
\bibliography{ICML_Model_Collapse_Citation.bib}

\clearpage
\appendix
\onecolumn
\IfFileExists{sections/math_app.tex}{\section{Mathematical Deductions and Proof of Theorems} \label{app:math}
We Adopt the notations from section \ref{sec:method4:setup}, for the Gram matrix $G^{(k)}$, if $\mathrm{rank}(M^{(k)}) < m$ (e.g., if $n_k < m$ or the columns of $M^{(k)}$ are linearly dependent), then $G^{(k)}$ is singular and $\det(G^{(k)})=0$. In practice, we usually have $n_k \gg m$, and the Cauchy–Binet formula for Gram matrices yields
\begin{align} \label{eq:CB}
    \det(G^{(k)})\ =\ \sum_{J\subset[n_k],\,|J|=m}\ \det(M^{(k)}_J)^2, 
\end{align}
where $M^{(k)}_J$ is the $m\times m$ submatrix formed by columns indexed by $J$. Recall from Section~\ref{sec:prelim} that computing the full Gram matrix can be computationally prohibitively expensive. Instead, we compute the log-determinant of the sub-Gram matrix, formed by selecting a subset of the columns in the embedding matrix.

\begin{proof}[Proof of Theorem \ref{thm:deterministic}]
To prove the lower bound, notice that $A^{(k)}$ is a sub-matrix of $M^{(k)}$. By the Cauchy Binet formula \eqref{eq:CB}, the sum defining $\det(G_A^{(k)})$ is a subset of the sum defining $\det(G^{(k)})$. This directly implies
\begin{align} \label{eq:CB-half_bound}
    \det(G_A^{(k)}) \leq& \det(G^{(k)}).
\end{align}
For the second half of the inequality, we use Weyl’s inequality and Ky Fan dominance and for each $i$:
\begin{align}
    \lambda_i(G^{(k)})\ \le\ \lambda_i(G_A^{(k)})+\lambda_1(G_B^{(k)}). 
\end{align}
Hence, with $\beta_k:=\lambda_{\max}(G_B^{(k)})$, we have:
\begin{align}
\label{eq:psd-majorant}
    G^{(k)} \preceq&\ G_A^{(k)}+\beta_k I_m,\quad \det(G^{(k)}) \leq \prod_{i=1}^m \left(\lambda_i(G_A^{(k)})+\beta_k\right).
\end{align}
Combining the lower bound in \eqref{eq:CB-half_bound} and the upper bound in \eqref{eq:psd-majorant} yields the stated result.
\end{proof}

The Cauchy-Binet formula, as well as the bound in Eq.~\eqref{eq:psd-majorant}, opens a window for us to bound the determinant of the original Gram matrix. Alas, the upper bound is not sharp because there is no way for us to extract information from the columns that we do not choose.

However, in practice, this is usually not the case because the dataset we are working on typically comes from a consistent context. If we use the analogy of viewing the Gram matrix as a book, then the chapters we read (i.e., $G_A^{(k)}$) are usually a good indication of the chapters we do not read (i.e., $G_B^{(k)}$), even though mathematically the unseen chapters could contain content that is completely irrelevant.

To reflect such realistic circumstances, we assume that the embedding column vectors are drawn from an implicit distribution (probability measure) $\mu$. Denote $v$ as an arbitrary column (where $v$ can be seen as a random variable) and let $C = \mathbb{E}[vv^{\top}]$ be the covariance matrix. With this additional assumption, We proceed with the asymptotic argument.

\begin{proof}[Proof of Theorem \ref{thm:stochastic}]
We leverage the concentration of the sample covariance matrix. Let $\hat{C}_{n_k} = \frac{1}{n_k} G^{(k)}$ and $\hat{C}_{n_A} = \frac{1}{n_A} G^{(k)}_A$. By the properties of determinant scaling, we have:
\begin{align}
    \det(G^{(k)}) &= \det(n_k \hat{C}_{n_k}) = n_k^m \det(\hat{C}_{n_k}), \nonumber \\
    \det(G^{(k)}_A) &= \det(n_A \hat{C}_{n_A}) = n_A^m \det(\hat{C}_{n_A}).
\end{align}
We consider the ratio:
\begin{align}
    \frac{\det(G^{(k)})}{\det(G^{(k)}_A)} = \left( \frac{n_k}{n_A} \right)^m \frac{\det(\hat{C}_{n_k})}{\det(\hat{C}_{n_A})}.
\end{align}
By the Strong Law of Large Numbers (or finite-sample concentration inequalities for covariance matrices), both $\hat{C}_{n_k}$ and $\hat{C}_{n_A}$ converge to the population covariance $C$ as $n_k, n_A \to \infty$.
Specifically, the continuous mapping theorem implies that $\det(\hat{C}_{n_k}) \xrightarrow{\mathbb{P}} \det(C)$ and $\det(\hat{C}_{n_A}) \xrightarrow{\mathbb{P}} \det(C)$. This implies that for any $\epsilon > 0$, with probability approaching 1:
\begin{align}
    1 - \epsilon \le \frac{\det(\hat{C}_{n_k})}{\det(\hat{C}_{n_A})} \le 1 + \epsilon.
\end{align}
Letting $K = 1+\epsilon$ establishes the upper bound with high probability:
\begin{align}
    \det(G^{(k)}) \le K \left( \frac{n_k}{n_A} \right)^m \det(G^{(k)}_A).
\end{align}
This confirms that the determinant of the full Gram matrix scales with the determinant of the sub-sample according to the geometric factor $(n_k/n_A)^m$, up to a stochastic constant $K$ that approaches 1.
\end{proof}

This bound serves as an asymptotic scaling law rather than a point-wise estimator of $\det(G^{(k)})$. It gives us a theoretical foundation for monitoring model collapse with partial data, assuming the bare minimum: that the data we are processing comes from a consistent context.

\begin{proof}[Proof of Theorem \ref{thm:err_stochastic}]
    Since we assume that the distribution of column vectors have a finite fourth moment, we know
    \begin{align} \label{eq:app:thm3:z}
        Z:=X^TC^{-1}X = \sum_{i,j=1}^m (C^{-1})_{ij}X_iX_j
    \end{align}
has a finite variance. i.e., $\sigma$ is finite. In fact, $Z$ is the Mahalanobis distance. It plays an important role in this proof. For lighter notation, we write $n \equiv n_k$ and  $G_t \equiv G_t^{(k)}$ the gram matrix formed by selecting $t$ columns from $M^{(k)}$. This way $G^{(k)} \equiv G_n$ and $G^{(k)}_A \equiv G_{n_A}$.

For arbitrary positive integer $N$, consider $\hat{C}_N := \frac{1}{N} \sum_{i=1}^N X_i X_i^\top$ the sample covariance matrix and by our notation, $G_N = N\hat{C}_N$. By the Central Limit Theorem for Sample Covariance Matrices (Theorem 3.4.4, \cite{Anderson_2003}), $\sqrt{N}(\hat{C}_N - C)$ converges in distribution to a Gaussian random matrix $M$ with mean zero:
\begin{align}
    \sqrt{N}(\hat{C}_N - C) \xrightarrow{d} M.
\end{align}
We know that $\hat{C}_N - C = O_p(N^{-1/2})$, and we shall see soon that what exactly is $M$ does not matter. Consider the function $f: \mathbb{S}_{++}^m \to \mathbb{R}$ defined by $f(S) = \log \det(S)$. This function is continuously differentiable at $S=C$, with the gradient given by $\nabla f(C) = C^{-1}$. Using Taylor expansion, we get
\begin{align}
    f(\hat{C}_N) = f(C) + \langle \nabla f(C), \hat{C}_N - C \rangle + R_N,
\end{align}
where $R_N$ is the remainder term. Since $\hat{C}_N - C = O_p(N^{-1/2})$, the remainder term is of order $O_p(N^{-1})$. Substituting the gradient and the inner product for matrices $\langle A, B \rangle = \text{tr}(A^\top B)$:
\begin{align} \label{eq:app:thm3:taylor}
    \log \det \hat{C}_N &= \log \det C + \text{tr}\left( C^{-1} (\hat{C}_N - C) \right) + O_p(N^{-1}).
\end{align}
By the definition of $\hat{C}_N$, the trace in \eqref{eq:app:thm3:taylor} can be written as
\begin{align} \label{eq:app:thm3:trace}
    \text{tr}\left( C^{-1} (\hat{C}_N - C) \right) &= \text{tr}\left( C^{-1} \left( \frac{1}{N} \sum_{i=1}^N X_i X_i^\top - C \right) \right) \nonumber \\
    &= \frac{1}{N} \sum_{i=1}^N \text{tr}(C^{-1} X_i X_i^\top) - \text{tr}(C^{-1} C) \nonumber \\
    &=\frac{1}{N} \sum_{i=1}^N \text{tr}(C^{-1} X_i X_i^\top) - m.
\end{align}
Now use the cyclic property of trace and combine \eqref{eq:app:thm3:taylor}, \eqref{eq:app:thm3:trace}, we arrive at
\begin{align}
    \log \det \hat{C}_N &= \log \det C + \frac{1}{N} \sum_{i=1}^N (X_i^\top C^{-1} X_i - m) + O_p(N^{-1}).
\end{align}
Inspired by the outcome and \eqref{eq:app:thm3:z}, define the scalar random variable $Z_i := X_i^\top C^{-1} X_i - m$. Since $X_i$ are i.i.d., $Z_i$ are also i.i.d. with mean $\mathbb{E}[Z_i] = 0$ and variance $\text{Var}(Z_i) = \sigma^2$. We can now rewrite the log-determinant of the Gram matrix $G_N = N \hat{C}_N$ as:
\begin{align} \label{eq:app:thm3:anchor}
    \log \det G_N = m \log N + \log \det C + \frac{1}{N} \sum_{i=1}^N Z_i + O_p(N^{-1}),
\end{align}
because the all matrices here are $m$-by-$m$, which explains the $m \log N$ term.
With \eqref{eq:app:thm3:anchor} being established, Consider the quantity of interest from Theorem 3:
\begin{align}
    \Delta_k := \log \det G_n - \log \det G_{n_A} - m \log \left(\frac{n_k}{n_A}\right).
\end{align}
Directly applying \eqref{eq:app:thm3:anchor} to $\Delta_k$ and we get
\begin{align}
    \Delta_k &= \left( \log \det C + \frac{1}{n_k} \sum_{i=1}^{n_k} Z_i \right) - \left( \log \det C + \frac{1}{n_A} \sum_{i=1}^{n_A} Z_i \right) + O_p(n_A^{-1}) \nonumber \\
    &= \sum_{i=1}^{n_k} \frac{Z_i}{n_k} - \sum_{i=1}^{n_A} \frac{Z_i}{n_A} + O_p(n_A^{-1}).
\end{align}
We then split the sum over $n_k$ into the observed part ($1 \dots n_A$) and the unobserved part ($n_A+1 \dots n_k$) so each summand is independent:
\begin{align}
    \Delta_k &= \left( \frac{1}{n_k} \sum_{i=1}^{n_A} Z_i + \frac{1}{n_k} \sum_{j=n_A+1}^{n_k} Z_j \right) - \frac{1}{n_A} \sum_{i=1}^{n_A} Z_i + O_p(n_A^{-1}) \nonumber \\
    &= \underbrace{\left( \frac{1}{n_k} - \frac{1}{n_A} \right) \sum_{i=1}^{n_A} Z_i}_{\text{Term 1}} + \underbrace{\frac{1}{n_k} \sum_{j=n_A+1}^{n_k} Z_j}_{\text{Term 2}} + O_p(n_A^{-1}).
\end{align}
Let $W_1 = \sum_{i=1}^{n_A} Z_i$ and $W_2 = \sum_{j=n_A+1}^{n_k} Z_j$.
By the standard Central Limit Theorem applied to i.i.d. variables with finite variance $\sigma^2$:
\begin{equation}
    \frac{W_1}{\sigma \sqrt{n_A}} \xrightarrow{d} \mathcal{N}(0, 1) \quad \text{and} \quad \frac{W_2}{\sigma \sqrt{n_k - n_A}} \xrightarrow{d} \mathcal{N}(0, 1).
\end{equation}
Let $Y$ be the random variable representing the limiting distribution of the leading terms of $\Delta_k$. Since $W_1$ and $W_2$ are independent (as the sample sets are disjoint), the linear combination of their limiting distributions is a Gaussian distribution with the sum of their variances.
The asymptotic variance is calculated as follows:
\begin{align}
    \sigma_{\Delta}^2 &= \text{Var}\left( \left( \frac{1}{n_k} - \frac{1}{n_A} \right) W_1 \right) + \text{Var}\left( \frac{1}{n_k} W_2 \right) \nonumber \\
    &= \left( \frac{1}{n_k} - \frac{1}{n_A} \right)^2 n_A \sigma^2 + \frac{1}{n_k^2} (n_k - n_A) \sigma^2 \nonumber \\
    &= \sigma^2 \left[ n_A \left( \frac{1}{n_k^2} - \frac{2}{n_A n_k} + \frac{1}{n_A^2} \right) + \frac{1}{n_k} - \frac{n_A}{n_k^2} \right] \nonumber \\
    &= \sigma^2 \left[ \frac{n_A}{n_k^2} - \frac{2}{n_k} + \frac{1}{n_A} + \frac{1}{n_k} - \frac{n_A}{n_k^2} \right] \nonumber \\
    &= \sigma^2 \left( \frac{1}{n_A} - \frac{1}{n_k} \right).
\end{align}
The remainder term $O_p(n_A^{-1})$ from the Taylor expansion converges to zero in probability relative to this standard deviation, which scales as $O(n_A^{-1/2})$.
Therefore, by Slutsky's Theorem, the normalized difference converges in distribution to a standard normal variable:
\begin{equation}
    \frac{\Delta_k}{\sigma \sqrt{\frac{1}{n_A} - \frac{1}{n_k}}} \xrightarrow{d} \mathcal{N}(0, 1).
\end{equation}
This completes the proof.
\end{proof}

\section{Mathematical Details for \sgma-UB (Section~\ref{subsec:method4:ubmetrics})}
\label{app:sigma_ub}

This appendix connects Theorems~\ref{thm:deterministic}--\ref{thm:stochastic} to the two computable
\sgma-UB diagnostics in Section~\ref{subsec:method4:ubmetrics} and justifies the normalizations used for
online monitoring.

\subsection{Regularized \texorpdfstring{$\log$-determinant}{log-determinant} and PSD monotonicity}
Fix $\delta>0$. For $X,Y\in\mathbb{S}^m_{+}$, if $X\preceq Y$ then $X+\delta I_m\preceq Y+\delta I_m$ and
\begin{align}
\log\det(X+\delta I_m)\le \log\det(Y+\delta I_m),
\label{eq:app:psd_monotone}
\end{align}
since $X+\delta I_m,  Y+\delta I_m\in\mathbb{S}^m_{++}$ for any $\delta>0$ and $\log\det(\cdot)$ is Loewner-monotone on $\mathbb{S}^m_{++}$. This is the basic tool that turns PSD majorants into one-sided envelopes for the (regularized) log-volume.

\subsection{Bridge I: deterministic envelope via a tail-energy budget (Track I)}
\label{app:sigmaub:bridgeI}

\subsubsection{From Theorem~\ref{thm:deterministic} to a PSD majorant}
Recall $G^{(k)}=G_A^{(k)}+G_B^{(k)}$ with $G_B^{(k)}\succeq 0$. By Weyl's inequality,
\begin{align}
\lambda_i\!\left(G^{(k)}\right)\le \lambda_i\!\left(G_A^{(k)}\right)+\lambda_1\!\left(G_B^{(k)}\right)
\qquad (i=1,\dots,m).
\end{align}
Let $\beta_k:=\lambda_{\max}(G_B^{(k)})=\lambda_1(G_B^{(k)})$. Then, in PSD order,
\begin{align}
G^{(k)} \preceq G_A^{(k)}+\beta_k I_m.
\label{eq:app:psd_majorant_beta}
\end{align}
Shifting by $\delta I_m$ and applying \eqref{eq:app:psd_monotone} gives the (incomputable) envelope
\begin{align}
\log\det\!\big(G^{(k)}+\delta I_m\big)
\le
\log\det\!\big(G_A^{(k)}+(\beta_k+\delta)I_m\big).
\label{eq:app:kf_envelope_beta}
\end{align}

\subsubsection{Replacing $\beta_k$ by a computable budget $\widehat{\beta}_k$}
Because $G_B^{(k)}\succeq 0$, we have the bound
\begin{align}
\beta_k=\lambda_{\max}(G_B^{(k)}) \le \mathrm{tr}(G_B^{(k)}).
\label{eq:app:beta_le_trace}
\end{align}
Using $G_B^{(k)}=\sum_{j\in B} v_j^{(k)}(v_j^{(k)})^\top$ gives
\begin{align}
\mathrm{tr}(G_B^{(k)})=\sum_{j\in B}\|v_j^{(k)}\|_2^2.
\label{eq:app:trace_sum_norms}
\end{align}
Assume preprocessing enforces $\|v_j^{(k)}\|_2^2\le \rho$ for all columns.
With $n_2:=n_k-n_A$ unseen columns,
\begin{align}
\beta_k \le \widehat{\beta}_k := n_2\rho .
\label{eq:app:betahat}
\end{align}
Combining \eqref{eq:app:psd_majorant_beta} and \eqref{eq:app:betahat} yields the computable PSD majorant
\begin{align}
G^{(k)} \preceq G_A^{(k)}+\beta_k I_m \preceq G_A^{(k)}+\widehat{\beta}_k I_m .
\label{eq:app:psd_majorant_hat}
\end{align}

\subsubsection{Valid regularized envelope and the normalized gain form}
Define the computable envelope
\begin{align}
\mathcal{U}_{\mathrm{KF}}^{(k)}(\delta)
:=
\log\det\!\left(G_A^{(k)}+(\widehat{\beta}_k+\delta)I_m\right).
\label{eq:app:kf_envelope_valid}
\end{align}
Then \eqref{eq:app:psd_majorant_hat} and \eqref{eq:app:psd_monotone} imply
\begin{align}
\log\det\!\big(G^{(k)}+\delta I_m\big)\le \mathcal{U}_{\mathrm{KF}}^{(k)}(\delta).
\end{align}

To interpret this envelope across checkpoints, we normalize out the isotropic baseline
$(\widehat{\beta}_k+\delta)I_m$, whose log-determinant equals $m\log(\widehat{\beta}_k+\delta)$. Define 
\begin{align}
\mathcal{G}_{\mathrm{KF}}^{(k)}(\delta)
&:= \mathcal{U}_{\mathrm{KF}}^{(k)}(\delta)-m\log(\widehat{\beta}_k+\delta).
\label{eq:app:gkf_def}
\end{align}
Equivalently,
\begin{align}
\mathcal{G}_{\mathrm{KF}}^{(k)}(\delta)
&=
\log\det\!\left(I_m+\frac{1}{\widehat{\beta}_k+\delta}G_A^{(k)}\right)
=
\sum_{i=1}^m \log\!\left(1+\frac{\lambda_i(G_A^{(k)})}{\widehat{\beta}_k+\delta}\right).
\label{eq:app:gkf_eigform}
\end{align}
Thus, $\mathcal{G}_{\mathrm{KF}}^{(k)}(\delta)$ defines a dimensionless gain relative to the tail-budget scale: it
decreases as the eigenvalues of $G_A^{(k)}$ contract, and it is most sensitive when eigenvalues are near the $\delta$ floor.

\subsubsection{Gap control relative to the ideal (incomputable) envelope}
Define the ideal but generally incomputable envelope as
\[
\mathcal{U}^{(k)}_{\mathrm{KF},\star}(\delta)
:=\log\det\!\left(G_A^{(k)}+(\beta_k+\delta)I_m\right).
\]
Since $\widehat{\beta}_k\ge \beta_k$, by monotonicity of the log-determinant,
\begin{align}
0 \le \mathcal{U}_{\mathrm{KF}}^{(k)}(\delta)-\mathcal{U}^{(k)}_{\mathrm{KF},\star}(\delta)
&= \sum_{i=1}^m
\log\!\left(
\frac{\lambda_i(G_A^{(k)})+\widehat{\beta}_k+\delta}
{\lambda_i(G_A^{(k)})+\beta_k+\delta}
\right)
\le m\log\!\left(\frac{\widehat{\beta}_k+\delta}{\beta_k+\delta}\right).
\label{eq:app:kf_gap_bound}
\end{align}
This implies
\begin{align}
\mathcal{U}_{\mathrm{KF}}^{(k)}(\delta)-m\log(\widehat{\beta}_k+\delta)
\le
\mathcal{U}^{(k)}_{\mathrm{KF},\star}(\delta)-m\log(\beta_k+\delta).
\label{eq:app:normalization_control}
\end{align}
Thus, the normalized quantity $\mathcal{G}_{\mathrm{KF}}^{(k)}(\delta)$ remains uniformly controlled by the
(unknown) ideal normalization, even when $\widehat{\beta}_k$ is a loose bound.

\subsection{Bridge II: LLN scaling and covariance normalization (Track II)}
\label{app:sigmaub:bridgeII}

\subsubsection{LLN scaling and cancellation of unknown constants by drift}
Motivated by the law-of-large-numbers (LLN) scaling suggested by Theorem~\ref{thm:stochastic}, define
\begin{align}
\mathcal{U}_{\mathrm{LLN}}^{(k)}(\delta)
:= \log\det(G_A^{(k)}+\delta I_m) + m\log\!\left(\frac{n_k}{n_A}\right).
\label{eq:app:ulln_def}
\end{align}
When the stochastic bound is written in logarithmic form as
\[
\log\det(G^{(k)}) \;\le\; \log K + \mathcal{U}_{\mathrm{LLN}}^{(k)}(0)
\quad \text{for an unknown finite }K,
\]
then baseline differencing removes $\log K$. Accordingly, define the differenced quantity
\begin{align}
\Delta \mathcal{U}_{\mathrm{LLN}}^{(k)}(\delta)
:= \mathcal{U}_{\mathrm{LLN}}^{(k)}(\delta)-\mathcal{U}_{\mathrm{LLN}}^{(0)}(\delta).
\label{eq:app:ulln_drift}
\end{align}

\subsubsection{Covariance-normalized form}
Define the covariance-normalized statistic
\begin{align}
\mathcal{U}_{\mathrm{LLN,cov}}^{(k)}(\delta)
&:= \mathcal{U}_{\mathrm{LLN}}^{(k)}(\delta)-m\log n_k
= \log\det(G_A^{(k)}+\delta I_m) - m\log n_A.
\label{eq:app:ulln_cov_def}
\end{align}
Equivalently, 
\begin{align}
\mathcal{U}_{\mathrm{LLN,cov}}^{(k)}(\delta)
=
\log\det\!\left(\frac{1}{n_A}G_A^{(k)}+\frac{\delta}{n_A} I_m\right).
\label{eq:app:ulln_cov_scaled}
\end{align}
Under the i.i.d.\ assumption, $\frac{1}{n_A}G_A^{(k)}\to C$ as $n_A\to\infty$, so for large $n_A$,
\[
\mathcal{U}_{\mathrm{LLN,cov}}^{(k)}(\delta)\approx \log\det(C),
\]
up to the vanishing regularization $\delta/n_A$. Hence $\mathcal{U}_{\mathrm{LLN,cov}}^{(k)}(\delta)$ is
size-invariant and is the natural object to difference over checkpoints for collapse monitoring.

\subsection{Numerical computation of \texorpdfstring{$\log$-determinant}{log-determinant}}
For $S\succ 0$, compute $\log\det(S)$ via Cholesky $S=LL^\top$:
\begin{align}
\log\det(S) = 2\sum_{i=1}^m \log L_{ii}.
\end{align}
Apply this to $G_A^{(k)}+\delta I_m$ and $G_A^{(k)}+(\widehat{\beta}_k+\delta)I_m$.

}{}
\IfFileExists{sections/experiment_app.tex}{\section{Supplementary Experiments, Reproducibility, and Extended Diagnostics (TECH)}
\label{app:exp_appendix}

\subsection{Section Roadmap}
This appendix section has three goals: (i) to make the TECH recursion loop fully reproducible from scratch, (ii) to make all main-text plots and scalars auditable from exported tables, and (iii) to provide extended diagnostics that contextualize geometric contraction.

\paragraph{(i) Reproducibility from raw sources to \sgma-UB curves.}
The recursion loop is sensitive to implementation choices. We therefore document the entire pipeline for the reported TECH run: (a) how the real TECH corpus and frozen prompt bank are constructed, (b) how synthetic pools are regenerated deterministically each generation, (c) how training mixtures are formed under a constant token budget, and (d) how evaluation responses, embeddings, and \sgma-UB metrics are computed. 

\paragraph{(ii) Plot--table consistency and “single source of truth.”}
All scalar summaries reported in the main text are computed \emph{directly} from exported per-generation tables. To prevent transcription errors, we include the exact plotted points as tables in Appendix~\ref{app:plot_tables_tech}.

\paragraph{(iii) Extended diagnostics.}
Because \sgma-UB is a \emph{geometric} diagnostic computed from embedding Gram structure, we also report two
\emph{surface-form baselines} on the \emph{same} generated outputs used to form the \sgma-UB embeddings
(Appendix~\ref{app:aux_baselines}). These baselines are included for interpretability: they quantify
lexical repetition directly, and help distinguish cases where geometric contraction aligns with (or departs from)
visible degeneracy in text.

\textbf{Baseline A: Distinct-$2$ \cite{li-etal-2016-diversity}.}
Distinct-$2$ is a standard diversity score defined as the ratio
\[
\mathrm{Distinct}\text{-}2 \;=\; \frac{\#\{\text{unique bigrams}\}}{\#\{\text{total bigrams}\}},
\]
computed on the generated corpus (or per-generation batch). Lower Distinct-$2$ indicates that the model
is reusing the same short phrases more frequently, i.e., increased short-range repetition.

\textbf{Baseline B: Hashed $n$-gram HHI \cite{CORMODE200558,article111} (global concentration of repeated phrases).}
To capture longer-range and heavier-tailed repetition, we compute a Herfindahl--Hirschman Index (HHI) over
$n$-gram frequencies: if $c_j$ is the count of $n$-gram type $j$ and $p_j=c_j/\sum_\ell c_\ell$, then
\[
\mathrm{HHI}_n \;=\; \sum_j p_j^2 .
\]
Higher HHI implies that a small number of $n$-grams account for a large share of the text, i.e., strong
concentration and mode collapse at the surface level. 

Finally, we provide an optional bucket-localization slice using the frozen TECH prompt taxonomy
(Appendix~\ref{app:bucket_localization}) to identify \emph{where} contraction is most pronounced.
Importantly, these surface-form baselines are \emph{contextual, not definitional}: they are reported to help
interpret when geometric contraction coincides with familiar repetition patterns, without redefining collapse
in terms of lexical features.

\subsection{TECH controlled assets: data provenance, licensing, and construction}
\label{app:tech_data_assets}

This paper’s experiments rely on two frozen assets for TECH: (i) a real corpus $R$ of fixed-length contexts used as prompts for synthetic regeneration and as the  real-data component of mixtures, and (ii) a frozen prompt bank $P$ used for evaluation.

\paragraph{Real TECH corpus $R$.}
We collect TECH documents from permissively licensed software and technical documentation sources (e.g., The Rust Book (MIT/Apache-2.0), Kubernetes documentation (CC BY 4.0), PostgreSQL documentation (BSD-style), SQLite documentation (Public Domain), and Python documentation (PSF license)), verifying license terms per source/page. Documents are converted to plain text with boilerplate removal and code-block stripping, filtered to English, and tokenized using the OPT-125M GPT-2 BPE tokenizer to match the base model’s vocabulary. We then extract 64-token contexts using a sliding window with fixed stride and enforce minimum-length and quality filters. Finally, we remove exact duplicates (content hashes) and near-duplicates prior to selecting the final $N=1000$ contexts.

\paragraph{Frozen TECH prompt bank $P$ (stable prompt IDs and buckets).}
Independently of $R$, we construct a TECH prompt bank designed to elicit a range of response regimes (open-ended creative prompts through neutral factual prompts) while maintaining evaluation stability across checkpoints. Prompts receive persistent prompt IDs and are categorized into buckets such as Creative, Divergent, Analogy/ELI5, What-if, and Neutral. 

\subsection{Reproducibility specification for S1 and S2}
\label{app:exp_card_tech}

We summarize every operational choice that affects outcomes: model hyperparameters; synthetic regeneration and mixture; decoding; seeds; \sgma-UB computation; evaluation metrics; hardware/software; and the formal difference between S1 and S2. Unless explicitly stated, S1 and S2 share the same configuration; the only intended difference is whether weights are reset to $M^{(0)}$ (S1) or carried forward $M^{(g-1)}$ (S2).

\begin{table*}[t]
\centering
\footnotesize
\setlength{\tabcolsep}{7pt}
\begin{tabular}{llll}
\toprule
Component & S1 & S2 & Notes \\
\midrule
Base model & \texttt{facebook/opt-125m} & \texttt{facebook/opt-125m} & 125M parameters \\
Fine-tuning epochs / gen & 5 & 5 & Fixed per-generation budget \\
Batch size & 16 & 16 & Effective batch size held constant \\
Learning rate & $2\times 10^{-5}$ & $2\times 10^{-5}$ & Constant across generations \\
Optimizer & AdamW & AdamW & Weight decay $0.01$ \\
Precision & FP16 enabled & FP16 enabled & GPU training \\
\bottomrule
\end{tabular}
\caption{\textbf{ Model and training hyperparameters.} These hyperparameters are held fixed across S1 and S2 so that observed differences are attributable to \emph{data recursion vs.\ data+weight recursion}, rather than to optimization budget changes.}
\label{tab:repro_hparams}
\end{table*}

\begin{table*}[t]
\centering
\footnotesize
\setlength{\tabcolsep}{5pt}
\begin{tabular}{llll}
\toprule
Item & Value & Applies to & Notes \\
\midrule
Real fraction $\alpha$ & $0.0$ (pure synthetic) & S1, S2 & Stress-tests recursion (no real-data anchoring) \\
Context unit & 64 tokens & S1, S2 & Training blocks; constant token budget \\
Real corpus size & $\sim$1{,}000--1{,}200 contexts / domain & S1, S2 & Domain-dependent (DecayBench) \\
Prompt bank size & $\sim$163--200 prompts / domain & Eval & Frozen prompt IDs and bucket labels \\
Generations reported & $g=0,\dots,50$ (51 checkpoints) & Main text & TECH run shown in Figures~\ref{fig:tech_s1_tracks}--\ref{fig:tech_s2_tracks} \\
Synthetic pool size & $|S^{(g)}|=|R|$ & S1, S2 & One synthetic continuation per real context per generation \\
\bottomrule
\end{tabular}
\caption{\textbf{ Data generation and mixing.} We use constant-budget mixtures and, for the reported runs, set $\alpha=0$ to isolate synthetic recursion dynamics without real-data refresh.}
\label{tab:repro_data_mixing}
\end{table*}

\begin{table*}[t]
\centering
\footnotesize
\setlength{\tabcolsep}{7pt}
\begin{tabular}{llll}
\toprule
Decoding parameter & Value & Used for & Notes \\
\midrule
Sampling type & Nucleus (top-$p$) & Gen + Eval & Stochastic sampling (not greedy) \\
Temperature & 1.0 & Gen + Eval & Fixed across generations \\
Top-$p$ & 0.95 & Gen + Eval & Fixed across generations \\
Max tokens (training) & 64 & Corpus construction & 64-token blocks for training \\
Max tokens (evaluation) & 256 & Prompt-bank responses & Used to obtain sufficiently long responses\\& & & for embedding/geometry \\
\bottomrule
\end{tabular}
\caption{\textbf{ Text generation (decoding) parameters.} Decoding is frozen across checkpoints to prevent evaluation drift and to keep the generation distribution comparable over time.}
\label{tab:repro_decoding}
\end{table*}

\begin{table*}[t]
\centering
\footnotesize
\setlength{\tabcolsep}{7pt}
\begin{tabular}{llll}
\toprule
\sgma-UB component & Value & Used for & Notes \\
\midrule
Sentence encoder & \begin{tabular}{@{}c@{}}\texttt{sentence-transformers/}\\ \texttt{all-MiniLM-L6-v2}\end{tabular} & Embedding & Frozen encoder, $m=384$ \\
Regularization $\delta$ & $10^{-3}$ & Track I/II & Matches Section~3.4 conventions \\
Energy cap $\rho$ & $1.0$ & Track I & Worst-case tail-energy budget \\
Baseline & $g=0$ & Drift computation & $\Delta(\cdot)$ reported relative to base checkpoint \\
Observed block & fixed $\mathcal{I}_A$ within run & Track I/II & Prevents measurement drift across checkpoints \\
\bottomrule
\end{tabular}
\caption{\textbf{ \sgma-UB computation.} Both tracks are computed from the same frozen embeddings and the same observed index set, so drift is attributable to model evolution rather than to evaluation sampling.}
\label{tab:repro_sigmaub}
\end{table*}

\begin{table*}[t]
\centering
\footnotesize
\setlength{\tabcolsep}{15pt}
\begin{tabular}{lll}
\toprule
Metric family & Reported quantity & Notes \\
\midrule
\sgma-UB Track II & $U_{\mathrm{LLN,cov}}$ and $\Delta U_{\mathrm{LLN,cov}}$ & Sensitive trend probe (observed-spectrum contraction) \\
\sgma-UB Track I & $G_{\mathrm{KF}}$ and $\Delta G_{\mathrm{KF}}$ & Conservative envelope; can be loose or non-monotone \\
Surface proxies & distinct-2, hashed $n$-gram HHI & Computed on the \emph{same} generated outputs as \sgma-UB \\
\bottomrule
\end{tabular}
\caption{\textbf{ Evaluation metrics.} The appendix focuses on geometry diagnostics and lightweight surface-form proxies; perplexity-based baselines are excluded by design.}
\label{tab:repro_metrics}
\end{table*}

\begin{table*}[ht!]
\centering
\footnotesize
\setlength{\tabcolsep}{7pt}
\begin{tabular}{p{0.25\textwidth}p{0.65\textwidth}}
\toprule
Component & Specification \\
\midrule
Training hardware & NVIDIA GPU (server environment) \\
Metrics hardware & CPU (Mac M1 used for aggregation/plotting) \\
Core libraries & PyTorch 2.0+, Transformers 4.30+, sentence-transformers 2.2+ \\
Precision & Mixed precision (FP16) enabled during training \\
\bottomrule
\end{tabular}
\caption{\textbf{ Computational environment.} Hardware does not change the \emph{definition} of metrics but can affect throughput; determinism is enforced via fixed seeds and fixed decoding.}
\label{tab:repro_env}
\end{table*}

\begin{table*}[ht!]
\centering
\footnotesize
\setlength{\tabcolsep}{7pt}
\begin{tabular}{p{0.22\textwidth}p{0.70\textwidth}}
\toprule
Prompt bucket & Description (frozen taxonomy) \\
\midrule
Creative & Open-ended creative prompts emphasizing novelty and stylistic variation \\
Divergent & Prompts encouraging multiple alternatives, brainstorming, or broad exploration \\
Analogy / ELI5 & Prompts asking for analogies, simplifications, or explanatory reframings \\
What-if & Counterfactual or perturbation prompts (``what if we change X?'') \\
Neutral & Factual/neutral prompts with minimal stylistic pressure \\
\bottomrule
\end{tabular}
\caption{\textbf{ Prompt buckets for localization.} Bucket IDs and assignments are fixed by DecayBench; slicing by bucket does not change the evaluation distribution over time, it only reveals where contraction is concentrated within a frozen taxonomy.}
\label{tab:repro_buckets}
\end{table*}

\paragraph{Operational end-to-end loop (implementation-level clarification).}
For auditability, we summarize the precise loop that produces one point on each drift curve:
(1) regenerate $S^{(g)}$ by prompting each $x\in R$ and sampling a fixed-length continuation with fixed decoding and per-context deterministic seeds;
(2) form $D^{(g)}(\alpha)$ under a constant block budget of $|R|$;
(3) fine-tune under either S1 (initialize from $M^{(0)}$) or S2 (initialize from $M^{(g-1)}$) using fixed hyperparameters;
(4) evaluate on the frozen prompt bank $P$ using fixed decoding and per-prompt deterministic seeds to produce $n_k$ responses;
(5) embed responses with the frozen encoder, select a fixed observed index set $\mathcal{I}_A$ once per run, and compute both \sgma-UB tracks and (optionally) surface proxies.
These steps are deterministic given the experiment card above and the fixed input assets.

\subsection{Exact per-generation values for the main-text drift plots}
\label{app:plot_tables_tech}

In particular, the S1 vs.\ S2 summary in Table~\ref{tab:tech_s1s2_summary} is computed directly from these per-generation series (including the baseline point $g=0$, where drift is $0$ by definition).

\begin{table*}[ht!]
\scriptsize
\setlength{\tabcolsep}{7pt}
\begin{minipage}[t]{0.4\textwidth}
\vspace{0pt}
\begin{tabular}{rcc}
\toprule
Generation $g$ & $\Delta U_{\mathrm{LLN,cov}}(\delta)$ (Track II) & $\Delta G_{\mathrm{KF}}(\delta)$ (Track I) \\
\midrule
0 & 0.000000 & 0.000000 \\
1 & -50.937706 & -42.771322 \\
2 & -87.220752 & -78.285600 \\
3 & -89.193112 & -80.257959 \\
4 & -108.940054 & -100.004901 \\
5 & -138.689852 & -130.138891 \\
6 & -135.146696 & -126.211544 \\
7 & -144.505214 & -135.954253 \\
8 & -158.731260 & -149.796107 \\
9 & -151.887859 & -142.952706 \\
10 & -149.844314 & -140.909162 \\
11 & -171.990854 & -163.055701 \\
12 & -181.172379 & -172.237226 \\
13 & -143.036471 & -134.101319 \\
14 & -146.954698 & -138.019546 \\
15 & -150.025009 & -141.089857 \\
16 & -169.047237 & -160.112084 \\
17 & -134.754288 & -125.819136 \\
18 & -116.984750 & -108.049597 \\
19 & -136.991334 & -128.056181 \\
20 & -135.678684 & -126.743532 \\
21 & -137.759104 & -128.823951 \\
22 & -150.287784 & -141.352631 \\
23 & -154.278070 & -145.342917 \\
24 & -132.792859 & -123.857707 \\
25 & -108.036399 & -99.101246 \\
\bottomrule%
\end{tabular}
\end{minipage}
\hspace{2cm}
\begin{minipage}[t]{0.4\textwidth}
  \vspace{0pt}
\begin{tabular}{rcc}
\toprule
Generation $g$ & $\Delta U_{\mathrm{LLN,cov}}(\delta)$ (Track II) & $\Delta G_{\mathrm{KF}}(\delta)$ (Track I) \\
\midrule
26 & -136.136651 & -127.201498 \\
27 & -137.363591 & -128.428438 \\
28 & -154.600200 & -145.665047 \\
29 & -134.832965 & -125.897812 \\
30 & -146.391209 & -137.456057 \\
31 & -152.558431 & -143.623278 \\
32 & -126.370288 & -117.435135 \\
33 & -166.712986 & -157.777834 \\
34 & -151.610056 & -142.674903 \\
35 & -158.777897 & -149.842744 \\
36 & -146.721644 & -137.786491 \\
37 & -171.229229 & -162.294077 \\
38 & -145.678374 & -136.743222 \\
39 & -134.396708 & -125.461556 \\
40 & -159.999869 & -151.064716 \\
41 & -151.415048 & -142.479895 \\
42 & -152.343583 & -143.408430 \\
43 & -136.593098 & -127.657945 \\
44 & -140.138527 & -131.203374 \\
45 & -163.640084 & -154.704932 \\
46 & -167.730336 & -158.795184 \\
47 & -138.267815 & -129.332663 \\
48 & -136.290231 & -127.355078 \\
49 & -174.970491 & -166.035338 \\
50 & -150.986188 & -142.051035 \\
\bottomrule%
\end{tabular}
\end{minipage}
\caption{
\textbf{Per-generation values for Figure~\ref{fig:tech_s1_tracks} (S1; restart-from-base).}
Track II and Track I are reported as drifts relative to the base checkpoint ($g=0$). In S1, both tracks remain consistently negative at the end of the run ($g=50$), indicating gradual contraction driven by the evolving synthetic pool even without weight carryover.
}
\label{tab:tech_s1_table}
\end{table*}

\begin{table*}[ht!]
\scriptsize
\setlength{\tabcolsep}{7pt}
\begin{minipage}[t]{0.4\textwidth}
\vspace{0pt}
\begin{tabular}{rcc}
\toprule
Generation $g$ & $\Delta U_{\mathrm{LLN,cov}}(\delta)$ (Track II) & $\Delta G_{\mathrm{KF}}(\delta)$ (Track I) \\
\midrule
0 & 0.000000 & 0.000000 \\
1 & -51.410418 & -1.041078 \\
2 & -70.477165 & -1.098910 \\
3 & -67.666663 & -1.110627 \\
4 & -102.965110 & -1.173303 \\
5 & -126.844686 & -1.220078 \\
6 & -156.320891 & -1.266556 \\
7 & -277.027401 & -1.665481 \\
8 & -482.600764 & -2.313280 \\
9 & -586.239077 & -2.729119 \\
10 & -589.243534 & -2.756645 \\
11 & -656.250424 & -2.917514 \\
12 & -409.708770 & -2.250243 \\
13 & -366.284030 & -2.065401 \\
14 & -326.883875 & -1.910236 \\
15 & -306.138619 & -1.815842 \\
16 & -266.006178 & -1.640532 \\
17 & -340.097542 & -1.858227 \\
18 & -459.780098 & -2.043834 \\
19 & -580.087915 & 1920.861134 \\
20 & -1222.621957 & 1011.736828 \\
21 & -1295.000015 & 925.612336 \\
22 & -1369.047226 & 906.107430 \\
23 & -985.967065 & 1455.366376 \\
24 & -993.200907 & 1426.878098 \\
25 & -1094.777825 & 1247.085610 \\
\bottomrule%
\end{tabular}
\end{minipage}
\hspace{2cm}
\begin{minipage}[t]{0.4\textwidth}
\vspace{0pt}
\begin{tabular}{rcc}
\toprule
Generation $g$ & $\Delta U_{\mathrm{LLN,cov}}(\delta)$ (Track II) & $\Delta G_{\mathrm{KF}}(\delta)$ (Track I) \\
\midrule
26 & -988.967085 & 1414.089742 \\
27 & -1033.645644 & 1337.676319 \\
28 & -1509.672539 & 671.355093 \\
29 & -1520.580241 & 409.488784 \\
30 & -1583.494588 & 323.527521 \\
31 & -1597.411170 & 297.547119 \\
32 & -1514.735790 & 589.234296 \\
33 & -1674.100026 & 139.234670 \\
34 & -1323.216571 & 78.225593 \\
35 & -304.629566 & 1.402128 \\
36 & -1284.962201 & 158.412491 \\
37 & -1415.459884 & 151.402919 \\
38 & -1426.475962 & 125.709896 \\
39 & -1773.662017 & 121.296272 \\
40 & -2036.524604 & 221.560576 \\
41 & -1983.942360 & 394.404189 \\
42 & -2191.980469 & 244.398064 \\
43 & -2020.929507 & 145.270029 \\
44 & -1994.987639 & 309.382624 \\
45 & -2028.916186 & 320.526442 \\
46 & -1955.248314 & -2.116992 \\
47 & -1829.693344 & 3.802502 \\
48 & -1885.721733 & 9.290424 \\
49 & -2000.020743 & -1.967390 \\
50 & -1536.562331 & -2.425766 \\
\bottomrule%
\end{tabular}
\end{minipage}
\caption{
\textbf{Per-generation values for Figure~\ref{fig:tech_s2_tracks} (S2; true recursion).}
Track II reaches large-magnitude negative drift (e.g., beyond $-1000$ by roughly $g\approx 20$ in this run), indicating rapid observed-spectrum contraction under weight carryover. Track I remains near zero and can be non-monotone (including occasional positive excursions), consistent with its role as a conservative envelope controlled by a tail-energy budget rather than by the observed spectrum itself.
}
\label{tab:tech_s2_table}
\end{table*}

\subsection{Surface-form proxy diagnostics and alignment with \sgma-UB}
\label{app:aux_baselines}

\sgma-UB is specifically constructed to follow the geometry of the gram matrix instead of merely capturing surface-level repetition, while still providing robust statistical validity and support for inference. Nevertheless, it is useful to ask whether geometry contraction coincides with familiar surface-form baseline metrics. We therefore compute two lightweight proxies on the \emph{same} generated text used for \sgma-UB embeddings: (i) distinct-2 (unique/total bigrams; higher is more diverse), and (ii) hashed $n$-gram HHI (higher is more concentrated/repetitive). These proxies are not collapse definitions and can be noisy, but they provide an operational sanity check.

\begin{figure*}[ht!]
  \centering
  \includegraphics[width=0.92\textwidth]{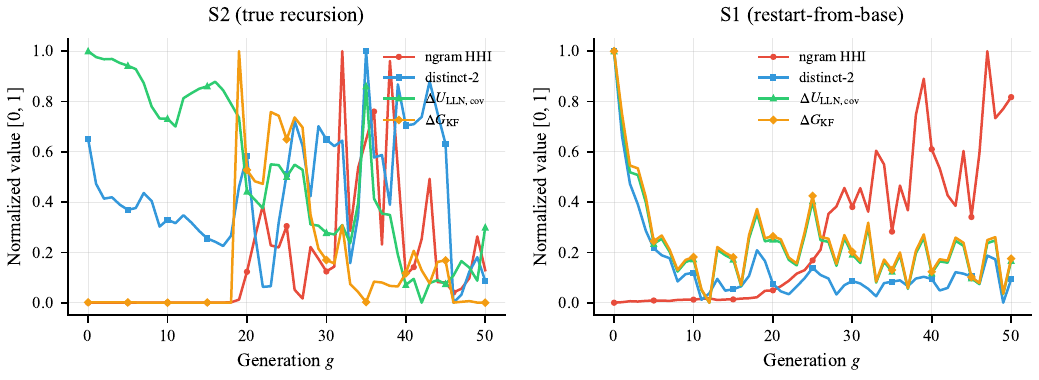}
  \caption{
  \textbf{Normalized overlay of \sgma-UB drifts and surface-form proxies (TECH).}
  Each curve is min--max normalized to $[0,1]$ within its panel for visual comparability. We overlay four signals across generations: Track II drift $\Delta U_{\mathrm{LLN,cov}}(\delta)$, Track I drift $\Delta G_{\mathrm{KF}}(\delta)$, distinct-2, and hashed $n$-gram HHI.
  \emph{Left:} S2 (true recursion). \emph{Right:} S1 (restart-from-base).
  The key qualitative trend is that the large early Track-II contraction in S2 precedes (and is accompanied by) stronger surface-form concentration (rising HHI) than in S1.
  }
  \label{fig:tech_surface_proxies}
\end{figure*}

Figure~\ref{fig:tech_surface_proxies} makes two practical points. First, in S2 the Track-II drift becomes strongly negative early, whereas hashed HHI rises more gradually and can spike later. This indicates that \sgma-UB Track II can act as an \emph{earlier} indicator of contraction than surface repetition alone. Second, Track I remains near the bottom of the normalized scale in S2 for much of the run, reinforcing that Track I is not intended as a sensitive trend detector; it is a conservative envelope that can remain loose even under clear contraction signals elsewhere.

To quantify the qualitative overlay, Table~\ref{tab:tech_surface_proxy_alignment} reports proxy endpoints ($g=0$ vs.\ $g=50$) and simple correlation diagnostics between Track II drift and the proxies over generations. The key operational takeaway is that true recursion (S2) yields substantially larger surface-form concentration than restart-from-base (S1), and this increase co-varies strongly with Track II under a rank-based measure.

\begin{table*}[ht!]
\centering
\footnotesize
\setlength{\tabcolsep}{2pt}
\begin{tabular}{lrrrrrrr}
\toprule
Setting &
distinct-2 ($g{=}0$) &
distinct-2 ($g{=}50$) &
HHI ($g{=}0$) &
HHI ($g{=}50$) &
$\mathrm{corr}(\Delta U,\mathrm{distinct}\text{-}2)$ &
$\mathrm{corr}(\Delta U,\mathrm{HHI})$ &
Spearman$(\Delta U,\mathrm{HHI})$ \\
\midrule
S1 & 0.731 & 0.334 & 0.000102 & 0.013385 & 0.926 & -0.286 & -0.318 \\
S2 & 0.715 & 0.254 & 0.000100 & 0.097262 & -0.165 & -0.468 & -0.685 \\
\bottomrule
\end{tabular}
\caption{
\textbf{Surface-proxy alignment with Track II (TECH).}
Distinct-2 decreases and hashed $n$-gram HHI increases in both settings, with substantially stronger concentration under true recursion (S2).
Correlations are computed over generations using the per-generation series (including $g=0$).
In S1, $\Delta U$ and distinct-2 co-move strongly (both decrease), yielding a high positive Pearson correlation.
In S2, distinct-2 is noisier (and can transiently increase despite strong contraction), while HHI retains a strong monotone association with $\Delta U$ (Spearman $\approx -0.69$), consistent with the overlay in Figure~\ref{fig:tech_surface_proxies}.
}
\label{tab:tech_surface_proxy_alignment}
\end{table*}

\begin{figure*}[ht!]
  \centering
  \includegraphics[width=0.85\textwidth]{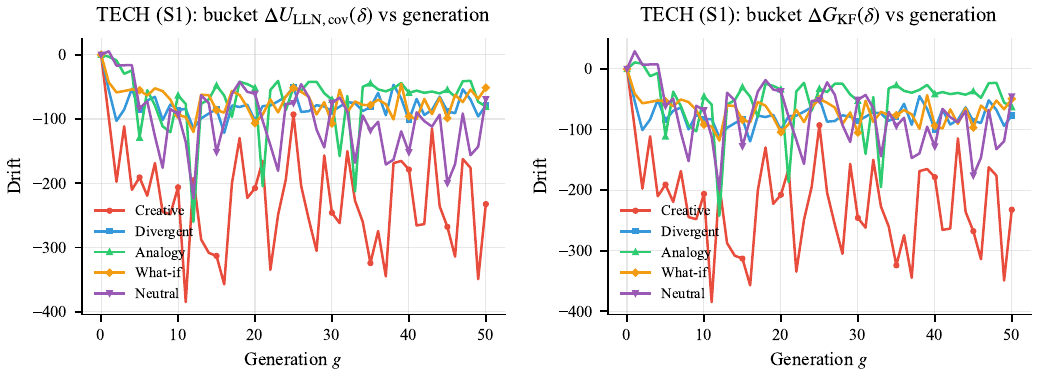}
  \caption{
  \textbf{Bucket localization under S1 (TECH; optional).}
  \sgma-UB drifts computed after restricting evaluation to each frozen prompt bucket.
  Bucket curves can be noisier than the aggregate because each bucket contains fewer prompts; we therefore interpret bucket localization as a qualitative ``where does collapse concentrate?'' tool rather than as a primary metric.
  }
  \label{fig:tech_bucket_localization}
\end{figure*}

\subsection{Within-domain localization by frozen prompt buckets }
\label{app:bucket_localization}

Because the evaluation prompt bank is annotated by prompt type (creative, divergent, analogy/ELI5, what-if, neutral), we can attribute model collapse to particular regimes while holding the evaluation framework constant over time. Specifically, we modularize the evaluation pipeline by limiting the framework to prompts from a single category, producing one response per prompt with identical decoding settings, embedding these responses using the same frozen encoder, and then computing the same \sgma-UB drift metrics.

\paragraph{Interpretation.}
Bucket localization helps answer a practical question: is model collapse uniformly distributed across prompt types, or concentrated where the model is pushed into higher-entropy regimes? In the reported S1 run, the \textsc{Creative} bucket shows consistently larger-magnitude Track-II contraction than the aggregate, suggesting that open-ended prompts amplify the representational collapse induced by synthetic recursion. That said, bucket series are noisier because each bucket contains fewer prompts and therefore smaller sample size.

\begin{table*}[ht!]
\centering
\scriptsize
\setlength{\tabcolsep}{6pt}
\begin{tabular}{rccccc}
\toprule
$g$ &
Creative &
Divergent &
Analogy/ELI5 &
What-if &
Neutral \\
\midrule
0 & 0.000000 & 0.000000 & 0.000000 & 0.000000 & 0.000000 \\
1 & -97.949065 & -55.159718 & -3.116045 & -42.928381 & 5.319456 \\
2 & -197.646061 & -102.878209 & -8.989381 & -58.633942 & -16.935859 \\
3 & -111.319080 & -85.428164 & -29.718339 & -56.452165 & -16.401600 \\
4 & -210.171845 & -52.805090 & -24.366967 & -53.958981 & -16.339483 \\
5 & -190.844854 & -86.526231 & -129.090423 & -55.162083 & -83.129598 \\
6 & -219.455934 & -71.041283 & -55.255710 & -63.533863 & -72.571280 \\
7 & -168.494531 & -65.564080 & -79.075972 & -52.309351 & -126.984507 \\
8 & -244.501580 & -101.711687 & -111.373454 & -56.608878 & -175.909325 \\
9 & -248.006326 & -77.422012 & -120.572586 & -69.700436 & -84.569910 \\
10 & -206.202956 & -85.442737 & -63.331140 & -93.868086 & -92.089470 \\
11 & -384.826592 & -84.560220 & -76.590153 & -96.869445 & -135.712099 \\
12 & -191.844812 & -114.666577 & -260.148545 & -120.239236 & -223.771048 \\
13 & -287.341241 & -99.117427 & -76.357705 & -62.108385 & -63.006851 \\
14 & -308.242176 & -91.501159 & -69.393513 & -63.105835 & -75.480431 \\
15 & -312.938055 & -84.708643 & -47.853525 & -86.176912 & -151.496944 \\
16 & -357.149831 & -121.594475 & -63.509059 & -90.049027 & -106.138078 \\
17 & -200.374851 & -78.959020 & -96.759571 & -56.453774 & -61.019787 \\
18 & -129.635475 & -81.399424 & -42.033181 & -62.702197 & -42.118850 \\
19 & -222.708576 & -77.688492 & -45.745484 & -81.863135 & -57.727289 \\
20 & -207.613162 & -102.415834 & -51.668224 & -106.350109 & -60.770080 \\
21 & -165.012269 & -79.875531 & -204.677032 & -93.085949 & -102.880295 \\
22 & -334.643279 & -79.037445 & -54.864907 & -69.461680 & -140.839753 \\
23 & -248.946713 & -72.418699 & -40.687391 & -89.783746 & -180.663248 \\
24 & -194.390032 & -65.316534 & -86.586005 & -64.861563 & -76.556536 \\
25 & -92.896077 & -50.135445 & -50.367304 & -51.637901 & -75.460966 \\
26 & -204.110634 & -88.983083 & -55.942330 & -57.759885 & -45.832592 \\
27 & -256.717613 & -87.868162 & -42.095798 & -65.592657 & -66.432976 \\
28 & -305.102122 & -78.531429 & -42.149740 & -92.922173 & -116.332023 \\
29 & -156.620419 & -80.583804 & -59.878911 & -73.651079 & -144.240864 \\
30 & -245.838960 & -88.216022 & -70.145223 & -107.060914 & -74.803194 \\
31 & -261.953853 & -80.824331 & -157.886675 & -54.458516 & -66.717892 \\
32 & -150.001326 & -73.853490 & -64.132656 & -89.356801 & -88.258118 \\
33 & -227.135310 & -75.126110 & -212.806264 & -72.637732 & -167.992962 \\
34 & -259.798832 & -87.601284 & -49.906544 & -78.191872 & -94.565098 \\
35 & -324.139874 & -81.799765 & -44.590636 & -78.285666 & -117.771632 \\
36 & -274.266735 & -59.455086 & -53.930531 & -69.795271 & -104.338444 \\
37 & -344.774439 & -94.113630 & -56.990416 & -76.943106 & -170.769428 \\
38 & -168.644537 & -46.805503 & -52.271075 & -100.768563 & -163.202054 \\
39 & -165.214096 & -68.054030 & -44.045319 & -45.225875 & -119.155531 \\
40 & -178.614273 & -101.818107 & -59.327405 & -95.767226 & -151.824765 \\
41 & -265.623322 & -68.819743 & -55.908253 & -100.310327 & -89.635536 \\
42 & -263.447148 & -93.701634 & -59.242546 & -68.744265 & -103.269554 \\
43 & -114.642464 & -85.956916 & -57.339421 & -92.521798 & -111.816243 \\
44 & -235.803806 & -65.343599 & -60.264987 & -66.722845 & -93.284889 \\
45 & -267.718776 & -87.940112 & -54.625358 & -98.657337 & -199.815871 \\
46 & -314.117762 & -91.077680 & -65.452508 & -62.377067 & -170.152065 \\
47 & -162.290329 & -53.746197 & -41.316406 & -73.824375 & -91.395436 \\
48 & -175.864254 & -69.564232 & -40.553158 & -54.144934 & -156.291398 \\
49 & -349.254658 & -96.036588 & -73.186376 & -68.534576 & -143.092325 \\
50 & -232.342407 & -78.865678 & -80.803164 & -51.253275 & -69.202759 \\
\bottomrule%
\end{tabular}
\caption{
\textbf{Track-II drift by bucket under S1 (TECH; optional).}
Each entry is $\Delta U_{\mathrm{LLN,cov}}(\delta)$ computed using only prompts from the given bucket (baseline $g=0$).
In this run, \textsc{Creative} tends to exhibit the largest-magnitude negative drift by late generations, while other buckets are comparatively less contracted.
Because bucket sizes differ, cross-bucket comparisons should be interpreted qualitatively unless normalized for prompt count.
}
\label{tab:tech_bucket_trackII}
\end{table*}

}{}

\end{document}